\documentclass[11pt]{article}

\usepackage[final]{acl}

\usepackage{times}
\usepackage{latexsym}
\usepackage{bbm}
\usepackage{amsthm}
\usepackage{amsmath}
\usepackage{amssymb}

\usepackage{algorithm}   
\usepackage{algpseudocode}
\usepackage{graphicx}
\usepackage{subcaption} 
\newtheorem{theorem}{Theorem}[section]    

\usepackage{multirow}
\usepackage{booktabs}
\algnewcommand\algorithmicbreak{\textbf{break}}
\algnewcommand\algorithmiccontinue{\textbf{continue}}
\algnewcommand\Break{\State \algorithmicbreak}
\algnewcommand\Continue{\State \algorithmiccontinue}
\algnewcommand{\To}{\textbf{to}}

\newenvironment{manualtheorem}[1]{
  \renewcommand{\thetheorem}{#1}
  \theorem 
}{
  \endtheorem 
}
\usepackage[T1]{fontenc}

\usepackage[utf8]{inputenc}

\usepackage{microtype}

\usepackage{inconsolata}

\usepackage{graphicx}

\title{Quantifying and Understanding Uncertainty in Large Reasoning Models}

\author{Yangyi Li\thanks{Equal contribution.}, Chenxu Zhao\footnotemark[1], Mengdi Huai \\
  Department of Computer Science,
  Iowa State University\\
  \texttt{\{liyangyi, cxzhao, mdhuai\}@iastate.edu} \\}

\begin{document}
\maketitle
\begin{abstract}








Large Reasoning Models (LRMs) have recently demonstrated significant improvements in complex reasoning. While quantifying generation uncertainty in LRMs is crucial, traditional methods are often insufficient because they do not provide finite-sample guarantees for reasoning-answer generation. Conformal prediction (CP) stands out as a distribution-free and model-agnostic methodology that constructs statistically rigorous uncertainty sets. However, existing CP methods ignore the logical connection between the reasoning trace and the final answer. Additionally, prior studies fail to interpret the origins of uncertainty coverage for LRMs as they typically overlook the specific training factors driving valid reasoning. Notably, it is challenging to disentangle reasoning quality from answer correctness when quantifying uncertainty, while simultaneously establishing theoretical guarantees for computationally efficient explanation methods. To address these challenges, we first propose a novel methodology that quantifies uncertainty in the reasoning-answer structure with statistical guarantees. Subsequently, we develop a unified example-to-step explanation framework using Shapley values that identifies a provably sufficient subset of training examples and their key reasoning steps to preserve the guarantees. We also provide theoretical analyses of our proposed methods. Extensive experiments on challenging reasoning datasets verify the effectiveness of the proposed methods.


\end{abstract}    
\section{Introduction}
\label{sec:intro}
\vspace{-1.5mm}
Recent advancements have further evolved toward Large Reasoning Models (LRMs), exemplified by the OpenAI o-series~\cite{openai2024learning} and DeepSeek-R1 series~\cite{guo2025deepseek}, capable of tackling complex tasks~\cite{shen-etal-2025-dast}. Unlike standard LLMs that directly map inputs to outputs, LRMs distinguish themselves by generating an explicit reasoning trace prior to the final answer. By scaling inference-time computations, these models acquire the capability to explore diverse solution paths, reflect on potential errors, and refine intermediate steps. This rigorous process allows LRMs to decompose intricate problems in challenging domains such as mathematics and coding, justifying their conclusions with logical derivations that reflect human problem-solving.

Ensuring the reliable deployment of such models in real-world applications requires accurate uncertainty quantification. However, conventional approaches typically lack rigorous statistical guarantees~\cite{tian-etal-2023-just}, limiting their trustworthiness in critical scenarios. While Conformal Prediction (CP)~\cite{vovk2005algorithmic,shafer2008tutorial} offers a distribution-free and model-agnostic framework to address this, traditional CP methods are insufficient for LRMs. Existing approaches either calibrate the final answer while ignoring the reasoning trace~\cite{ye2024benchmarking}, or treat the generation as a whole without verifying whether the reasoning logically supports the final answer~\cite{quach2024conformal}. Consequently, by failing to account for the intrinsic reasoning-answer structure of LRMs, these methods cannot guarantee that a correct answer is supported by a valid reasoning trace. This limitation necessitates the design of a valid uncertainty quantification framework explicitly tailored to this special structure.

Beyond establishing valid uncertainty sets, interpreting the origins of quantified uncertainty is equally critical. For LRMs, identifying which training examples and steps are sufficient to guarantee uncertainty coverage provides essential insights into the reliability of the reasoning trace and guides further model refinement. However, existing literature largely overlooks the connection between uncertainty quantification and its explanation. Standard explanation methods~\cite{selvaraju2017grad} typically focus on label prediction and lack the mechanisms to trace the validity of uncertainty sets back to the training stage. Moreover, they often treat training samples as single units, failing to pinpoint which reasoning steps within an example are critical for ensuring that the model generates valid reasoning-answer pairs. Consequently, even when uncertainty is rigorously guaranteed, the training factors behind this reliability remain opaque.

However, addressing these limitations presents several unique challenges. First, formulating a rigorous uncertainty guarantee for valid reasoning-answer pairs is inherently difficult. Specifically, it requires disentangling the quality of the reasoning trace from the correctness of the final answer while simultaneously capturing their interdependence within a rigorous statistical framework. Second, attributing valid prediction uncertainty to specific training data faces significant computational hurdles. This complexity escalates dramatically when extending attribution to the fine-grained level of reasoning steps, as rigorous attribution typically necessitates evaluating model performance across an exponentially growing number of data subsets. While approximation methods are required to mitigate these costs, establishing rigorous theoretical guarantees for such approximations presents a critical challenge. Specifically, it is challenging to certify that the extracted data subset identified through approximation is sufficient to preserve the statistical risk control of the original model.

To address the aforementioned challenges, in this paper, we first propose \textbf{Co}nformal \textbf{R}easoning-\textbf{A}nswer \textbf{P}rediction (\textbf{CoRAP}), a novel framework designed to jointly quantify uncertainty of the reasoning-answer structure. By defining distinct quality functions to evaluate logical interdependence of reasoning-answer pairs and employing a rigorous statistical calibration procedure, CoRAP constructs uncertainty sets with finite-sample guarantees, ensuring that the set contains a generated sequence in which the true answer is supported by a valid reasoning trace with a user-specified probability. In addition, we propose a unified example-to-step explanation framework using Shapley values to identify the most influential training factors contributing to the uncertainty coverage. To resolve computational intractability, we employ a hierarchical Monte Carlo approximation that first identifies pivotal training examples and subsequently isolates critical reasoning steps. Crucially, our framework extracts an example or step subset that is provably sufficient to ensure the uncertainty guarantees. Furthermore, we provide detailed theoretical analysis to certify the validity of our proposed uncertainty quantification and explanation methods. Extensive experiments across complex reasoning tasks demonstrate the effectiveness of our approach in ensuring theoretically guaranteed and interpretable uncertainty quantification.
\vspace{-1mm}
\section{Related Work}
\label{sec:related_work}
\vspace{-1.5mm}
Large Reasoning Models (LRMs) represent a significant evolution by generating explicit reasoning traces to decompose intricate problems~\cite{an2025dont}. Leveraging this capability, they play a pivotal role in various applications, including mathematical reasoning~\cite{zhang-2025-confidence,minegishi2025topology}, web search~\cite{li2025webthinker}, and code generation~\cite{li2025teaching}. While uncertainty quantification is crucial for ensuring model trustworthiness~\cite{qian2025towards_cikm}, existing work developed for LLMs remains inadequate for LRMs~\cite{han-etal-2025-attributes}. Heuristic approaches~\cite{zhang-zhang-2025-cot,li-huai-2025-quantifying} relying on token-level probabilities or verbalized confidence scores, fundamentally lack rigorous statistical guarantees. Conversely, frameworks employing CP offer finite-sample coverage guarantees~\cite{li2024data,wang2025sample} but fail to account for the intrinsic reasoning-answer structure of LRMs. Specifically, existing approaches either calibrate the final answer while ignoring the intermediate reasoning trace~\cite{ye2024benchmarking}, or treat the generation process as a whole without verifying whether the reasoning logically supports the final answer~\cite{quach2024conformal,su2025cp}. Consequently, these methods often ensure coverage for the output but disregard the logical validity of the reasoning.

Moreover, explaining the sources of uncertainty is equally critical for transparency. Traditional approaches~\cite{watson2023explaining,slack2021reliable,qian2024towards,chen2024modeling,li2026uncertainty} attribute uncertainty to static input features but are intractable for LRMs due to the prohibitive cost of long-sequence sampling and the complexity of dynamic reasoning. While natural language explanations are widely used to interpret LLM predictions~\cite{kumar-talukdar-2020-nile}, they often fail to account for uncertainty. To address this, \cite{liu-etal-2025-abstain} attempts to categorize uncertainty sources to generate inquiries. However, this heuristic approach lacks both the step-level granularity essential for complex reasoning and rigorous statistical guarantees for explanation validity. In this work, we address these limitations by proposing a framework that jointly quantifies the uncertainty of the reasoning-answer structure with finite-sample guarantees. To identify influential factors, we introduce a unified example-to-step explanation method based on Shapley values. We resolve computational intractability via a hierarchical Monte Carlo approximation that efficiently isolates pivotal training examples and reasoning steps. Crucially, our framework extracts a subset provably sufficient to maintain the established uncertainty guarantees.


\vspace{-1mm}
\section{Methodology}
\vspace{-1.5mm}
\label{sec:method}
In this section, we introduce our proposed framework designed to quantify and explain uncertainty in LRMs. While this task is critical, existing methods fail to disentangle reasoning traces from answers to guarantee the validity of reasoning-answer pair uncertainty, and they lack computationally feasible mechanisms to attribute uncertainty to specific data sources with guarantees. To address these challenges, we first propose a novel method that constructs statistically valid uncertainty sets for reasoning-answer pairs. Subsequently, we describe our example-to-step explanation framework, which utilizes Shapley values to identify the specific training examples and reasoning steps responsible for ensuring valid uncertainty quantification.

Formally, we consider a supervised reasoning task using a training set $\mathcal{D}_{\mathrm{tr}}=\{z_j = (x_j,q_j,r_j,y_j)\}_{j=1}^{n_{\mathrm{tr}}}$, where each instance consists of an input image $x_j$, a query $q_j$, a reasoning trace $r_j$ including steps, and a final answer $y_j$. We denote a complete answer as $a = r \Vert y$. An LRM parameterizes a generation policy $\pi_\theta$ that models the joint probability of reasoning and answer tokens conditioned on the input. We use $\hat{a}$ to denote an answer generated by the model.

\vspace{-1mm}
\subsection{Conformal Reasoning-Answer Prediction}
\vspace{-1.5mm}
\label{sec:cra}

Here, we propose a post-hoc and model-agnostic method that assigns provably valid uncertainty to the reasoning-answer structure given a calibration set $\mathcal{D}_{\mathrm{cal}} = \{z_i = (x_i, q_i, r_i, y_i)\}_{i=1}^{n_\mathrm{cal}}$. The challenge is to disentangle the reasoning trace $r$ from the final answer $y$, and to guarantee that a correct answer is supported by a valid reasoning process.

To address this challenge, we first define distinct quality functions designed to disentangle the evaluation of reasoning and answers. Specifically, the plausibility of a full response is measured by the sequence quality function
$Q(x_i,q_i,\hat{a}_i)=\frac{1}{|\hat{a}_i|}\log p_\theta(\hat{a}_i | x_i,q_i)$.
We define the set confidence function as
$F(\mathcal{C})=\max_{a\in\mathcal{C}} Q(x_i,q_i,a)$. Furthermore, to validate the interdependence between the reasoning trace and the final answer, a conditional answer quality function $A(x_i, q_i, \hat{a}_i)
= p_\theta(\hat{y}_i | x_i,q_i,\hat{r}_i)$ is employed to evaluate the model's confidence in $\hat{y}_i$ conditioned on $\hat{r}_i$. Intuitively, the sequence quality function $Q$ filters out low-quality individual sequences, the set confidence function $F$ ensures at least one candidate in the set is high-confidence, and the answer quality function $A$ acts as a consistency check to ensure the answer is grounded in the reasoning trace.

We employ a sampling-based procedure to construct the prediction set by iteratively expanding a candidate pool. At each step $k$, we sample a sequence $\hat{a}_k$ from the model and include it in the output set $\mathcal{C}$ only if its sequence quality $Q$ exceeds the threshold $\lambda_1$. This accumulation continues until the step $K \le K_{\mathrm{max}}$, at which point two termination criteria are simultaneously met: (i) the set confidence $F(\mathcal{C}_{ K}) \ge \lambda_2$, and (ii) the answer quality $A( x_i, q_i,a^*) \ge \lambda_3$ for the highest-scoring candidate $a^*$ in the set $\mathcal{C}_{ K}$. Considering a finite grid $\Lambda$ of threshold tuples $\lambda=(\lambda_1,\lambda_2,\lambda_3)$, the final output set is denoted as $\mathcal{C}^{\mathrm{RA}}_{\lambda}(x_i, q_i;\theta)$. 



To assess the validity of the reasoning-answer pair uncertainty with finite-sample guarantees, we adopt a specialized loss function within a Learn Then Test~\cite{angelopoulos2025learn} framework to identify the specific threshold configuration $\lambda$ that ensures uncertainty guarantees on unseen data. The core idea is to define a failure event for any given instance. We now calibrate which threshold configuration to use so that the probability of such failures is controlled on test data. Specifically, the loss function is formulated to explicitly capture the failure of generating a correct answer supported by valid reasoning as follows:
\vspace{-2mm}
\begin{align}
L^{\mathrm{RA}}(z_i;\lambda,\theta)
=&\mathbf{1}\{\nexists\,\hat{a}_k \in \mathcal{C}^{\mathrm{RA}}_{\lambda}(x_i,q_i;\theta):\\
& V(z_i,\hat{r}_k)=1 \text{ and } \hat{y}_k=y_i\}, \notag
\end{align}

\vspace{-3mm}
\noindent where $V(z_i,\hat{r}_k)$ is a binary admission function that returns 1 if the generation reasoning trace is admitted for the input $x_i$, and 0 otherwise. 
Given a model $\theta$, we evaluate every candidate $\lambda\in\Lambda$ on the calibration set by computing its empirical risk. To ensure statistical validity, we convert the estimated risk into a statistical $p$-value and apply family-wise error rate (FWER) control at level $\varepsilon$. The valid set $\Lambda_{\mathrm{valid}}(\theta)$ retains only configurations with sufficiently small FWER-adjusted $p$-values. With probability at least $1-\varepsilon$, the true risk of every retained configuration is bounded by $\alpha$. If this set is empty, the procedure abstains. Otherwise, we select a final $\hat{\lambda}(\theta)$ from $\Lambda_{\mathrm{valid}}(\theta)$ that minimizes the expected set size. 


\begin{theorem}
\label{thm:CoRAP_ltt}
Assume calibration examples in $\mathcal D_{\mathrm{cal}}$ and a test sample $z_t$ are i.i.d., and CoRAP's sampling randomness is independent across examples.
Let $\Lambda_{\mathrm{valid}}(\theta)$ be obtained by applying an FWER-$\varepsilon$ procedure to the $p$-values
for target level $\alpha\in(0,1)$ on model $\theta$.
Then, for any
$\hat\lambda(\theta)\in\Lambda_{\mathrm{valid}}(\theta)$, we have the risk constraint
\vspace{-1mm}
\begin{align}
\label{eq:risk_control}
\mathbb E[L^{\mathrm{RA}}(z_t;\hat\lambda(\theta),\theta)] \le \alpha,
\end{align}

\vspace{-1mm}
\noindent with probability at least $1-\varepsilon$ over $\mathcal D_{\mathrm{cal}}$.
\end{theorem}
\vspace{-1mm}

In other words, this theorem guarantees that with probability at least $1-\varepsilon$ over the calibration data, the expected loss on $z_t$ is strictly controlled below the user-specified target $\alpha$. Since the loss indicates the failure to retrieve a valid reasoning-answer pair, this risk bound implies that the constructed uncertainty set covers a correct answer supported by valid reasoning with a probability of at least $1-\alpha$. We provide the detailed proof in Appendix~\ref{sec:proof}.

\vspace{-1mm}
\subsection{Example-to-Step Explanations for Conformal Reasoning-Answer Prediction}
\vspace{-1mm}
\label{sec:explanation_framework}
Based on the sets returned by CoRAP, we aim to identify the most influential training examples and reasoning steps that are crucial to satisfy the uncertainty guarantees in the theorem above for a given test instance. The primary challenge lies in the computational intractability of exact attribution methods, which typically require exponential retraining. Moreover, employing approximation methods introduces the critical obstacle of ensuring that the identified subset is statistically sufficient to ensure the rigorous risk control.

To overcome the challenges above, we develop a unified data valuation framework based on Shapley values~\cite{shapley1953value} with guarantees. At the example level, we quantify how each training example in $\mathcal{D}_{\mathrm{tr}}$ contributes to covering $y_{t}$ and select an appropriate subset $S^\star_{\mathrm{ex}}$ that is provably sufficient for coverage at $x_{t}$. At the step level, conditioned on the selected examples $S^\star_{\mathrm{ex}}$, we further attribute contribution to individual reasoning steps inside these examples and extract a step subset $T^\star$ that remains sufficient to achieve coverage. The two levels together produce an example-to-step explanation of CoRAP success with valid guarantees.

\textbf{Formalizing the explanation problem.}
We define the contributing players at two levels. At the example level, let $\mathcal{P}_{\mathrm{ex}}=[n_\mathrm{tr}]$, where each player $j$ corresponds to a training example. At the step level, we define the player set $\mathcal{P}_{\mathrm{st}}$ as the set of all steps from the selected examples $S_{\mathrm{ex}}^\star$.
Next, we define value functions that measure the success of any coalition. Given a coalition of examples $S \subseteq \mathcal{P}_{\mathrm{ex}}$, we train a model $\theta^S$ exclusively on the examples within $S$. Subsequently, we execute the CoRAP procedure at level $\alpha$ using this model to generate the uncertainty set $\mathcal{C}^{\mathrm{RA}}_{\hat{\lambda}}(x_{t},q_{t};\theta^S)$. The example-level value is defined to measure whether a subset $S$ suffices to ensure valid coverage:
\vspace{-1mm}
\begin{align}
\label{eq:vex}
\!\!\!v_{\mathrm{ex}}(S; z_{t}):=1-L^{\mathrm{RA}}(z_t;\hat{\lambda}(\theta^S),\theta^S).
\end{align}

\vspace{-1.5mm}
Conditioned on $S_{\mathrm{ex}}^\star$, we consider per-example step selections. For the step universe $\mathcal{P}_{\mathrm{st}}$ derived from $S_{\mathrm{ex}}^\star$, consider a step coalition $T \subseteq \mathcal{P}_{\mathrm{st}}$. The model $\theta^{T}$ is trained on the restricted steps in $T$, while all other steps are discarded. Running CoRAP under this training regime yields the set $\mathcal{C}^{\mathrm{RA}}_{\hat{\lambda}}(x_{t},q_{t};\theta^{T})$. From this, we define the conditional value to quantify the step-level contribution towards satisfying the guarantee:
\vspace{-1mm}
\begin{align}
\label{eq:vst}
&v_{\mathrm{st}}(T ; z_{t}):=1- L^{\mathrm{RA}}(z_t;\hat{\lambda}(\theta^{T}),\theta^{T}).
\end{align}

\vspace{-1mm}
\noindent We specifically aim to identify the pivotal subset of training examples or steps that is sufficient to ensure the validity of the test sample by value functions. Overall, this hierarchical formulation allows us to first identify the influential examples and subsequently isolate the critical reasoning steps within them, effectively avoiding the computational overhead of searching the entire step universe.

\begin{figure*}[t] 
  \centering
  \begin{subfigure}[t]{0.246\textwidth}
    \centering
    \includegraphics[width=\linewidth]{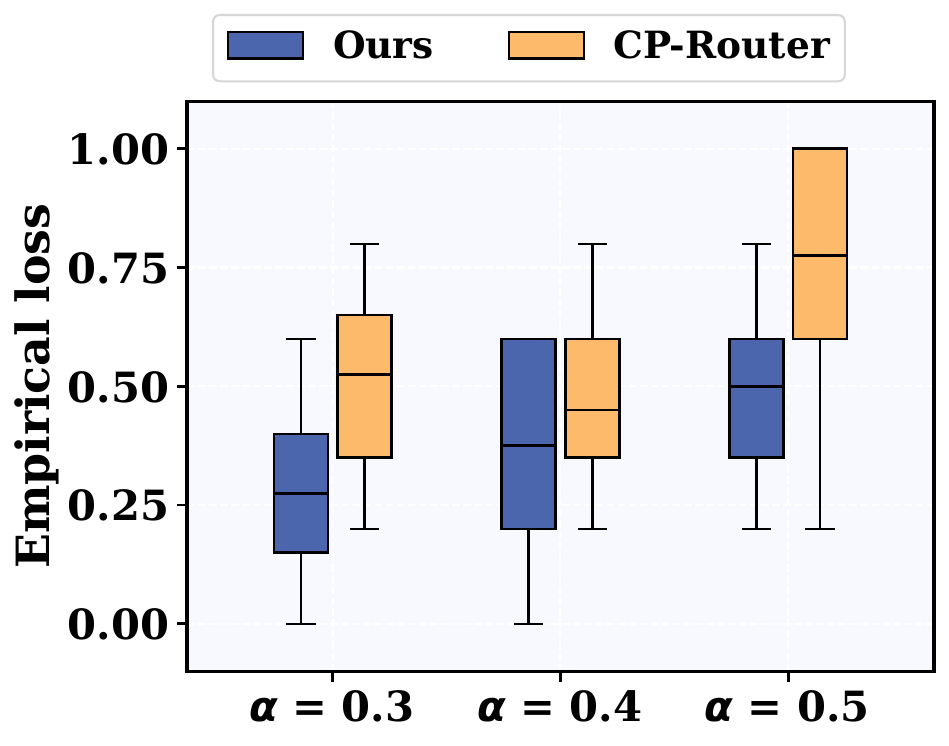}
    \vspace{-6mm}
    \caption{Validity on CLEVR-Math}
    \label{fig:cp_validity_clevr}
  \end{subfigure}\hfill
  \begin{subfigure}[t]{0.246\textwidth}
    \centering
    \includegraphics[width=\linewidth]{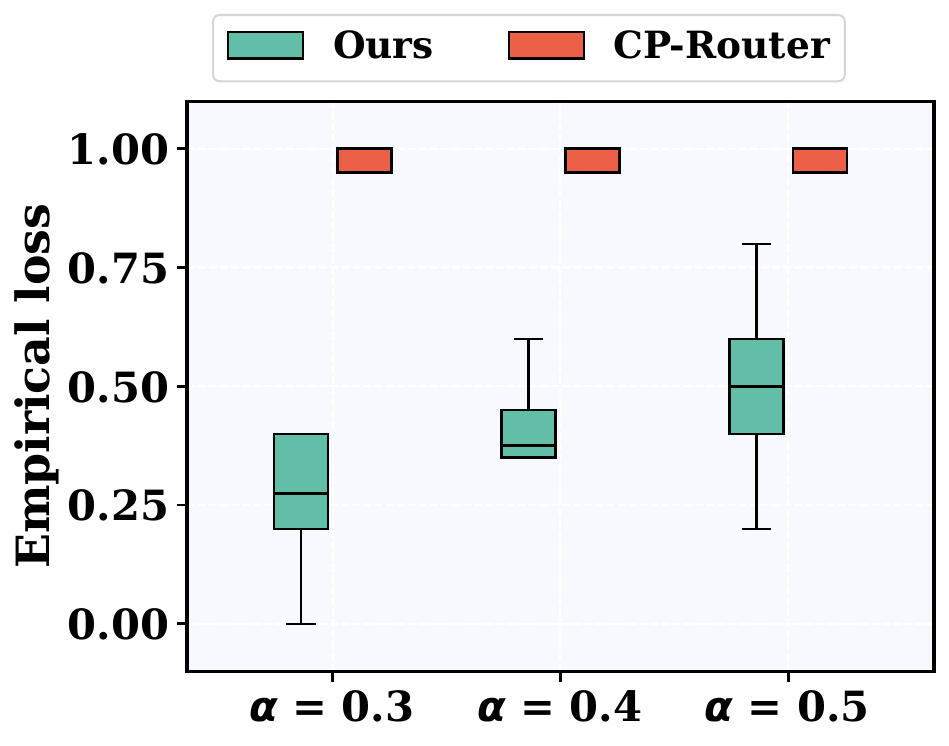}
    \vspace{-6mm}
    \caption{Validity on ScienceQA}
    \label{fig:cp_validity_sqa}
  \end{subfigure}\hfill
    \begin{subfigure}[t]{0.25\textwidth}
    \centering
    \includegraphics[width=0.92\linewidth]{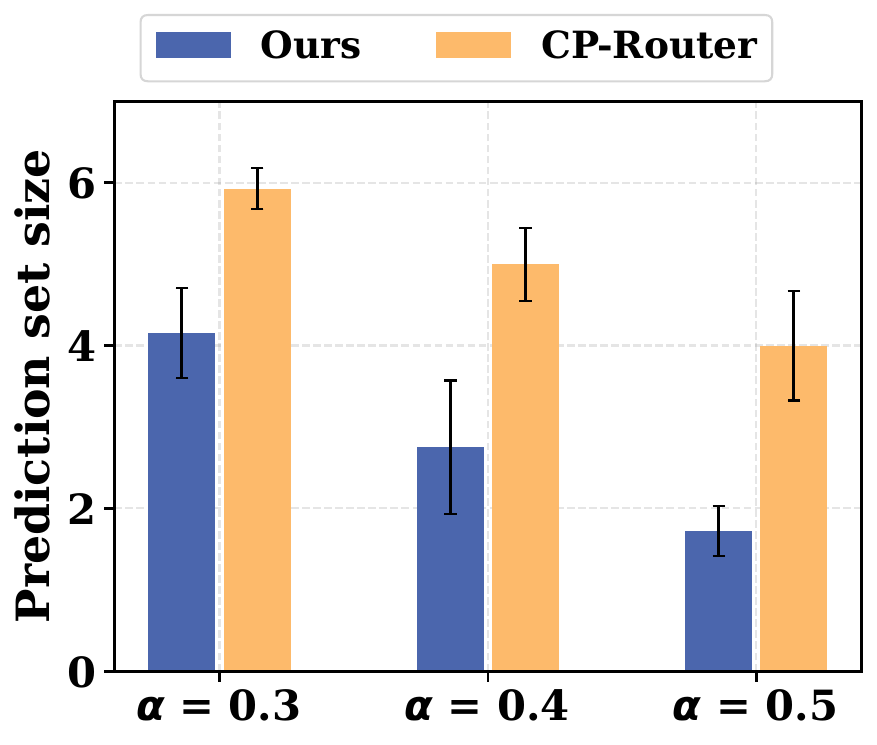}
    \vspace{-1.35mm}
    \caption{Efficiency on CLEVR-Math}
    \label{fig:cp_eff_clevr}
  \end{subfigure}\hfill
  \begin{subfigure}[t]{0.24\textwidth}
    \centering
    \includegraphics[width=\linewidth]{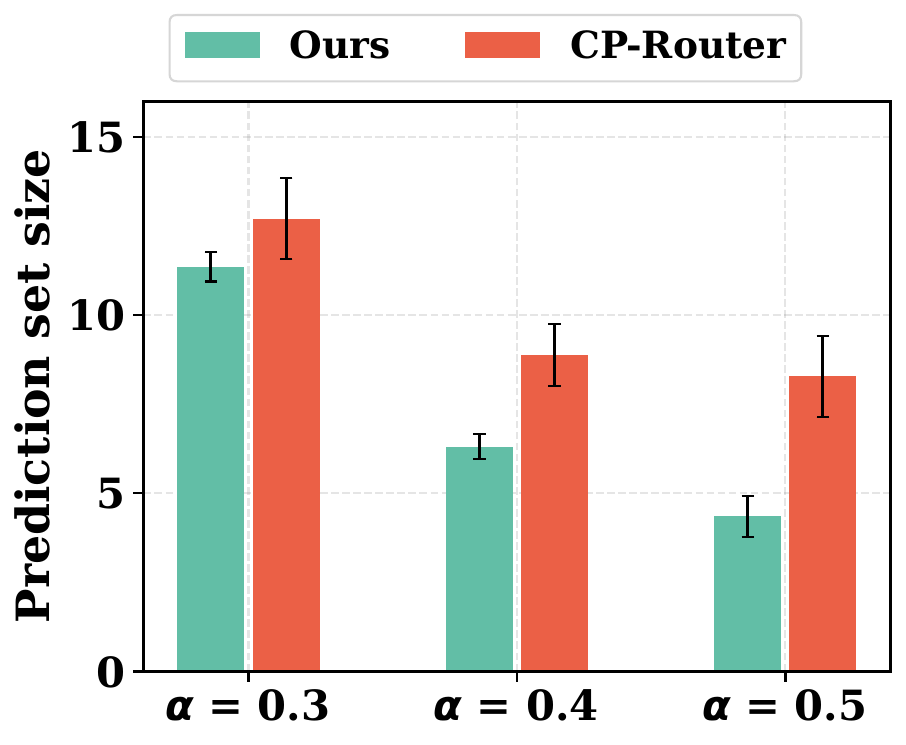}
        \vspace{-6mm}
    \caption{Efficiency on ScienceQA}
    \label{fig:cp_eff_sqa}
  \end{subfigure}

  \vspace{-2mm}
  \caption{Empirical loss and efficiency results across datasets.}
  \label{fig:conformalresults}
  \vspace{-4mm}
\end{figure*}
\textbf{Quantifying contribution and constructing provable explanations.}
To identify a subset sufficient to satisfy the coverage guarantee, we quantify the contribution of each player (an example or a step) using Shapley values. Conceptually, this game-theoretic metric measures a player's importance as its average marginal contribution to the coalition's value across all permutations. Since computing the exact Shapley value requires summing over exponentially many coalitions, it is computationally intractable. To address the computational challenge, we employ a Monte Carlo (MC) approximation. We sample $M$ independent and uniformly random permutations $\tau_1,\ldots,\tau_M$ of $\mathcal{P}$. For each $\tau_m$, let $S_u^{(m)}$ denote the players appearing before $u$ in the sequence. The unbiased estimator for the Shapley value is defined as:
\vspace{-1mm}
\begin{align}
\label{eq:shapley_mc}
\widehat{\phi}_u=\frac{1}{M}\sum_{m=1}^{M}(v(S_u^{(m)}\cup\{u\})-v(S_u^{(m)})).
\end{align}

\vspace{-1mm}
\noindent To quantify statistical reliability, we leverage the fact that the marginal contribution $v(S\cup\{u\})-v(S)$ is bounded in $[-1,1]$. 
Applying Hoeffding's inequality and a union bound over a player universe $\mathcal{P}$, with probability at least $1-\delta$ we have simultaneously for all $u\in\mathcal{P}$ that
$\phi_u \ge \widehat{\phi}_u - b(M,\delta,|\mathcal{P}|)$, where $b(M,\delta,|\mathcal{P}|) = 2\sqrt{\log(2|\mathcal{P}|/\delta)/(2M)}$
represents the estimation uncertainty radius. We define $\mathcal{B}_u := \widehat{\phi}_u - b(M,\delta,|\mathcal{P}|)$ as the lower confidence bound. Intuitively, $\mathcal{B}_u$ serves as a conservative estimate of a player's true importance, penalizing the empirical mean $\widehat{\phi}_u$ by the sampling noise to prevent overestimation. Then we can construct our provable explanation by finding a subset of players whose cumulative certified contributions are sufficient to ensure the risk control.

At the example level, we compute $\widehat{\phi}^{\mathrm{ex}}_u$ on $\mathcal{P}_{\mathrm{ex}}$ using $m_{\mathrm{ex}}$ permutations and a failure budget $\delta_{\mathrm{ex}}$. We order the players by non-increasing lower confidence bounds $\mathcal{B}^{\mathrm{ex}}_{u}$ to prioritize the most reliably influential examples. To identify a sufficient subset, we select the smallest number of top-ranked examples $k$ such that their cumulative certified contribution satisfies $\sum_{i=1}^k \mathcal{B}^{\mathrm{ex}}_{u_i} \ge 1-\alpha$, and report this set as $S^\star_{\mathrm{ex}}$.
Conditioned on the selected examples $S^\star_{\mathrm{ex}}$, we refine the explanation to the step level using the value function $v_{\mathrm{st}}$. We compute $\widehat{\phi}^{\mathrm{st}}_u$ on the step universe $\mathcal{P}_{\mathrm{st}}$ with $m_{\mathrm{st}}$ permutations and budget $\delta_{\mathrm{st}}$. Similarly, we sort the steps by non-increasing $\mathcal{B}^{\mathrm{st}}_u$ and select the smallest $k$ steps necessary to meet the condition $\sum_{i=1}^k \mathcal{B}^{\mathrm{st}}_{u_i} \ge 1-\alpha$. The resulting set of reasoning steps is reported as $T^\star$.




\vspace{-1mm}
\begin{figure*}[t] 
  \centering
  \includegraphics[width=\textwidth]{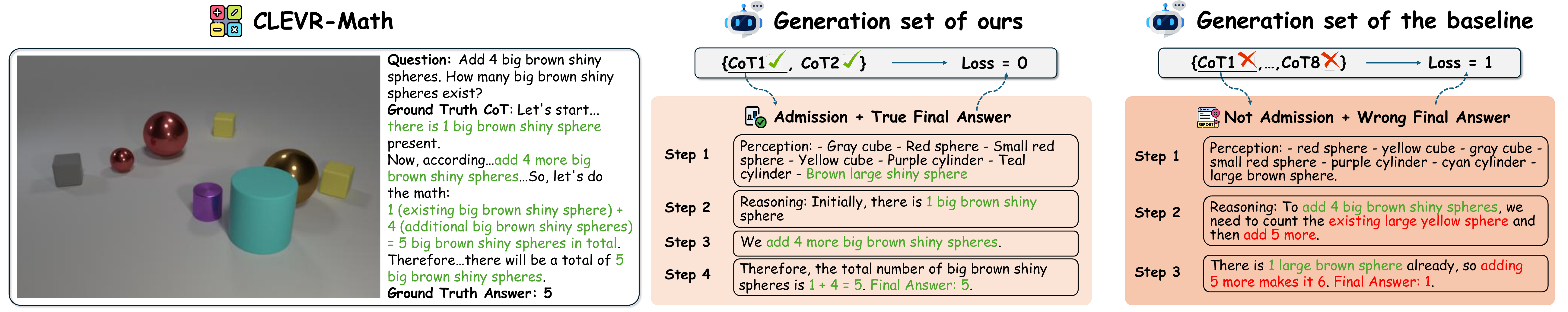}
  \vspace{-6mm}
  \caption{Visualization of our proposed conformal prediction framework on CLEVR-Math.}\vspace{-4mm}
  \label{fig:example}
\end{figure*}

\begin{theorem}
\label{thm:two_level_exp}
Let $z_t$ be a test example that is i.i.d. with the calibration examples in $\mathcal{D}_{\mathrm{cal}}$, and let
$\mathcal H := \{\,L^{\mathrm{RA}}(z_t;\hat\lambda(\theta_0),\theta_0)=1\,\}$
denote the event that the base model $\theta_0$ fails on $z_t$.
Condition on the validity of the thresholds used in evaluating
$v_{\mathrm{ex}}$ and $v_{\mathrm{st}}$.
Assume that for a target level $\alpha \in (0,1)$ and error bounds $\xi_{\mathrm{ex}}, \xi_{\mathrm{st}} \ge 0$, the condition
$
v(\mathcal P\setminus U) \ge v(\mathcal P) -\sum_{u\in U}\phi_u - \xi
$
holds for $(v,\mathcal P,\xi)=(v_{\mathrm{ex}},\mathcal P_{\mathrm{ex}},\xi_{\mathrm{ex}})$
and
$(v,\mathcal P,\xi)=(v_{\mathrm{st}},\mathcal P_{\mathrm{st}},\xi_{\mathrm{st}})$.

\noindent\textbf{(i) (Example-level).}
With probability at least $1-\delta_{\mathrm{ex}}$ over the Shapley MC, the example subset $S^\star_{\mathrm{ex}}$
identified by the stopping condition $\sum_{u\in S^\star_{\mathrm{ex}}}\mathcal B^{\mathrm{ex}}_{u} \ge 1-\alpha$
is sufficient to achieve the risk control:
\begin{align}
\label{eq:vex_control}
\mathbb{E}[
L^{\mathrm{RA}}(z_t;\hat\lambda(\theta^{S^\star_{\mathrm{ex}}}),\theta^{S^\star_{\mathrm{ex}}})
\mid \mathcal H
]
\ \le\ \alpha+\xi_{\mathrm{ex}}.
\end{align}

\vspace{-1mm}
\noindent\textbf{(ii) (Step-level).}
Conditionally on $S^\star_{\mathrm{ex}}$, with probability at least $1-\delta_{\mathrm{st}}$ over the step-level Shapley MC,
the step subset $T^\star$ identified by the stopping condition $\sum_{u\in T^\star}\mathcal B^{\mathrm{st}}_{u} \ge 1-\alpha$
is sufficient to achieve the risk control:
\begin{align}
\label{eq:vst_control}
\mathbb{E}[
L^{\mathrm{RA}}(z_t;\hat\lambda(\theta^{T^\star}),\theta^{T^\star})
\mid \mathcal H
]
\ \le\ \alpha+\xi_{\mathrm{st}}.
\end{align}

\end{theorem}

\noindent Theorem~\ref{thm:two_level_exp} provides a two-level guarantee on our unified explanation framework.
Part (i) certifies the example-level explanation by stating that if the cumulative certified contribution of the selected examples $S^\star_{\mathrm{ex}}$ meets the threshold $1-\alpha$, the subset is sufficient to ensure the risk control in Eq.~\eqref{eq:vex_control} with probability $1-\delta_{\mathrm{ex}}$.
This implies that a model trained solely on $S^\star_{\mathrm{ex}}$ satisfies the valid coverage requirements established in Theorem~\ref{thm:CoRAP_ltt}.
Part (ii) extends this certification to the step-level explanation given $S^\star_{\mathrm{ex}}$.
Specifically, provided that the cumulative certified contribution of the selected steps $T^\star$ exceeds the threshold $1-\alpha$, the resulting model trained exclusively on the reasoning steps in $T^\star$ within $S^\star_{\mathrm{ex}}$ is guaranteed to achieve the risk control in Eq.~\eqref{eq:vst_control} with probability $1-\delta_{\mathrm{st}}$.
We also provide the detailed proof in Appendix~\ref{sec:proof}.

\textbf{Discussion.} We analyze the time complexity for a single test instance and propose several strategies to further reduce the computational overhead. At the example level, the total baseline complexity is $O (m_{\mathrm{ex}}(n_\mathrm{tr}^2+n_\mathrm{tr} K_{\mathrm{max}}) + m_{\mathrm{st}}(|\mathcal{P}_{\mathrm{st}}|^2+|\mathcal{P}_{\mathrm{st}}|K_{\mathrm{max}})  + |\Lambda|\,n_{\mathrm{cal}}K_{\mathrm{max}} ).$
By leveraging influence functions and machine unlearning techniques~\cite{zhao2024rethinking,chen2025survey}, we can substitute computationally expensive retraining with efficient approximate updates. With a one-time setup cost $I$ (e.g., gradient caching), the complexity reduces to $O (I + m_{\mathrm{ex}}(n_\mathrm{tr}+n_\mathrm{tr} K_{\mathrm{max}})+ m_{\mathrm{st}}(|\mathcal{P}_{\mathrm{st}}|+|\mathcal{P}_{\mathrm{st}}|K_{\mathrm{max}}) + |\Lambda|\,n_{\mathrm{cal}}K_{\mathrm{max}} ).$
To further optimize, we partition the $n_{\mathrm{tr}}$ players into $G$ disjoint groups. We employ a warm-start preselection with $m_0$ permutations to identify one active group by the value function $v = 1$, restricting the intensive estimation to this sole group. The final optimized complexity is $O(I + m_0(G+GK_{\mathrm{max}}) + m_{\mathrm{ex}}K_{\mathrm{max}}+ 
m_{\mathrm{st}}(|\mathcal P_{\mathrm{st}}|+|\mathcal P_{\mathrm{st}}|K_{\mathrm{max}})
+ |\Lambda|\,n_{\mathrm{cal}}K_{\mathrm{max}}
),$
which yields significant savings when the average group size $\bar g = n_{\mathrm{tr}} / G \ll n_{\mathrm{tr}}$. This optimization effectively decouples the inference-time cost from the quadratic scaling of the training data, enabling the practical application of our method to large-scale reasoning models.

\vspace{-1mm}
\section{Experiments}
\vspace{-1.5mm}
In this section, we conduct extensive experiments to evaluate the effectiveness of our proposed methods. More experimental details and results (e.g., experimental results about uncertainty on various models) are deferred to the Appendix of the paper.

\textbf{Datasets and models.} In experiments, we evaluate our method on two challenging multimodal reasoning datasets. CLEVR-Math~\cite{lindstrom2022clevr} assesses mathematical reasoning in visual contexts, while ScienceQA~\cite{lu2022learn} provides a diverse set of science-domain questions based on images. We follow the preprocessing pipeline from \cite{tan2025reason}. We consider various LRMs, including 
LLaVA-CoT~\cite{xu2025llava}, 
LMM-R1~\cite{peng2025lmm}, and R1-Onevision~\cite{yang2025r1} models.

\textbf{Baselines.}
For conformal prediction, we compare against \textit{CP-Router}~\cite{su2025cp}, a recent uncertainty method tailored to LRMs.
For explanation and data attribution methods in LRMs, to the best of our knowledge there is no dedicated prior method.
We therefore include a \textit{Random} baseline that uniformly samples a subset of examples (or steps), treats them as important,  and fine-tunes under identical training and inference budgets.

\textbf{Implementation details.} In experiments, we use a temperature of 1.2 and top-p sampling of 0.85 to produce candidate sequences, where the sampling budget $K_{max} = 16$ per input. For each dataset, we sample 1500 calibration examples from the original training corpus and 100 test examples from the test set, using the remaining training examples as the training data. All experiments are run for 8 trials, and we report the averaged results.





\vspace{-1mm}
\subsection{Conformal Prediction for LRMs}
\vspace{-1mm}

\textbf{Validity.} We evaluate the validity of our proposed uncertainty quantification framework on large reasoning models by constructing valid uncertainty sets for reasoning-answer pairs. Experiments are conducted on CLEVR-Math and ScienceQA using LMM-R1, under target significance levels $\alpha = 0.3, 0.4,$ and $0.5$. As illustrated in Fig.~\ref{fig:cp_validity_clevr} and~\ref{fig:cp_validity_sqa}, the empirical losses of our methods remain below the target level~$\alpha$, demonstrating that our method satisfies Eq.~\eqref{eq:risk_control} across various $\alpha$ and thus upholds the validity guarantee about the reasoning-answer structure established in Theorem~\ref{thm:CoRAP_ltt}. In contrast, CP-Router shows empirical losses exceeding the target levels in multiple settings, indicating that it fails to maintain the desired risk control.

\textbf{Efficiency.} In Fig.~\ref{fig:cp_eff_clevr} and~\ref{fig:cp_eff_sqa}, we examine the efficiency of our uncertainty quantification framework by prediction set size across different LRMs. Experiments are conducted on CLEVR-Math and ScienceQA using LMM-R1 under significance levels $\alpha = 0.3, 0.4,$ and $0.5$, maintaining consistency with the validity evaluation. As shown, our proposed method consistently produces compact uncertainty sets over reasoning steps compared with CP-Router, enabling efficient interpretation of the model’s thought process. For instance, on CLEVR-Math, our method achieves an average set size of only $2.76$ with LMM-R1 at $\alpha=0.4$, meaning that fewer than three reasoning traces are typically sufficient to cover the correct answer while retaining valid coverage guarantees. In contrast, CP-Router achieves an average set size of $5.00$ in the same setting. Fig.~\ref{fig:example} visually demonstrates this efficiency on a CLEVR-Math example. Our method generates a compact set containing only two reasoning traces. As shown in the green-coded section, these traces pass the admission function and achieve a true final answer, resulting in zero loss. In contrast, CP-Router generates a large set with eight traces. Its example trace, containing severe reasoning flaws highlighted in red, is rejected by the admission function and produces a wrong final answer, leading to a loss of 1.
These combined results show that our method efficiently quantifies uncertainty by producing compact and high-quality sets of reasoning traces, making it a practical tool for analyzing complex visual reasoning tasks.
\begin{table}[t]
  \centering
  \small 
  \setlength{\tabcolsep}{1pt} 
  \begin{tabular}{llcc}
    \toprule
    Model & Method & Top-1 Shapley value & Success rate \\
    \midrule
    \multirow{2}{*}{LMM-R1} & Random & $0.042\pm 0.026$ & $0.000\pm 0.000$ \\
     & Ours & $\mathbf{0.433\pm 0.116}$ & $\mathbf{0.833\pm 0.167}$ \\
    \midrule
    \multirow{2}{*}{R1-Onevision} & Random & $0.021\pm 0.038$ & $0.000\pm 0.000$ \\
     & Ours & $\mathbf{0.356\pm 0.058}$ & $\mathbf{1.000\pm 0.000}$ \\
    \bottomrule
  \end{tabular}\vspace{-2mm}
    \caption{Data-level explanations of the uncertainty sets.}
  \label{tab:data_X}\vspace{-1mm}
\end{table}

\begin{table}[t]
  \centering
  \small 
  \setlength{\tabcolsep}{1pt} 
  \begin{tabular}{llcc}
    \toprule
    Model & Method & Top-1 Shapley value & Success rate \\
    \midrule
    \multirow{2}{*}{LMM-R1} & Random & $0.020\pm 0.016$ & $0.000\pm 0.000$ \\
     & Ours & $\mathbf{0.739\pm 0.128}$ & $\mathbf{0.875\pm 0.125}$ \\
    \midrule
    \multirow{2}{*}{R1-Onevision} & Random & $0.006\pm 0.005$ & $0.000\pm 0.000$ \\
     & Ours & $\mathbf{0.589\pm 0.119}$ & $\mathbf{0.750\pm 0.164}$ \\
    \bottomrule
  \end{tabular}\vspace{-2mm}
    \caption{Step-level explanations of the uncertainty sets.}
  \label{tab:step_X}\vspace{-6mm}
\end{table}

\vspace{-1.5mm}
\subsection{Explanations of Prediction Sets in LRMs}
\vspace{-1mm}

\textbf{Data-level explanations.} As shown in Table \ref{tab:data_X}, our proposed example-level explanations substantially outperform a random retrieval baseline on CLEVR-Math using LMM-R1 and R1-Onevision when $\alpha=0.5$, with the value of top-1 increasing from 0.042 to 0.433 using LMM-R1 and from 0.021 to 0.356 using R1-Onevision. Using the selected data to fine-tune changes the outcome of the model, so the uncertainty set for the test data shifts from excluding the true final answer to including it on both models, while fine-tuning on randomly chosen examples does not achieve this. This demonstrates that our method accurately identifies pivotal training evidence via Shapley value, providing a suitable explanation for why the uncertainty set ultimately includes the true final answer. Furthermore, to empirically validate the theoretical guarantee in Theorem \ref{thm:two_level_exp} (i), we set the example-level failure parameter $\delta_{\mathrm{ex}} = 0.25$. The success rate in Table \ref{tab:data_X}, which measures the empirical frequency of successfully identifying a sample subset sufficient to maintain the CoRAP risk control, is $0.833$ for LMM-R1 and $1.000$ for R1-Onevision. Both results are greater than the required level $1-\delta_{\mathrm{ex}}$, empirically verifying the effectiveness of our proposed example-level explanation guarantee.


\textbf{Step-level explanations.} As shown in Table \ref{tab:step_X}, our proposed step-level explanations substantially outperform a random step-selection baseline on CLEVR-Math using LMM-R1 and R1-Onevision when $\alpha=0.5$, with the top-1 Shapley value of steps increasing from 0.020 to 0.739 using LMM-R1 and from 0.006 to 0.589 using R1-Onevision. Similar to the data-level analysis, we further empirically validate the theoretical guarantee for step-level attribution in Theorem \ref{thm:two_level_exp} (ii). Setting the step-level failure probability $\delta_{\mathrm{st}} = 0.25$, the success rate in Table \ref{tab:step_X} is $0.875$ for LMM-R1 and $0.750$ for R1-Onevision. Both results surpass the required threshold $1-\delta_{\mathrm{st}} = 0.75$, confirming the validity of our step-level explanation guarantee. These results present the influence of intermediate reasoning traces. We find that the uncertainty set's shift from excluding to including the true final answer is driven by critical reasoning steps within the model's generation process. Our method accurately identifies these pivotal internal reasoning steps by Shapley value, thereby providing suitable explanations for why the uncertainty set of test data ultimately includes the true final answer at the step-level for LRMs. 

\begin{figure}[t] 
    \centering
    \begin{subfigure}[t]{0.48\columnwidth} 
        \centering
        \includegraphics[width=\linewidth]{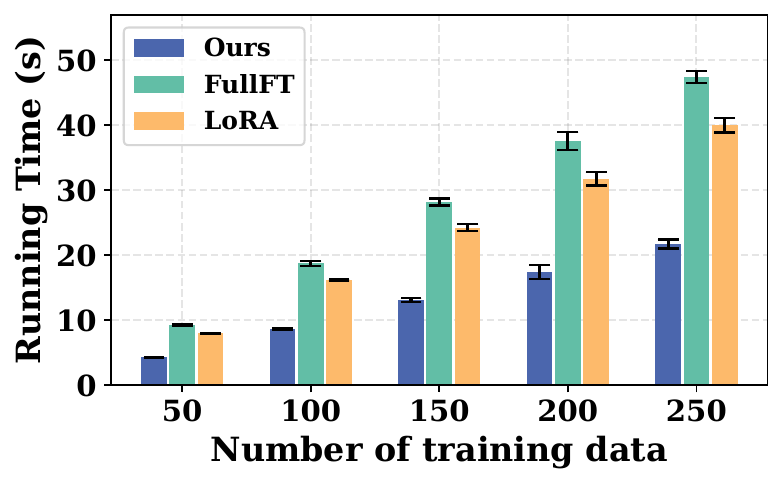} 
        \vspace{-6.5mm}\caption{Fine-tuning methods} 
        \label{fig:time_left} 
    \end{subfigure}
    \hfill 
    \begin{subfigure}[t]{0.48\columnwidth} 
        \centering
        \includegraphics[width=\linewidth]{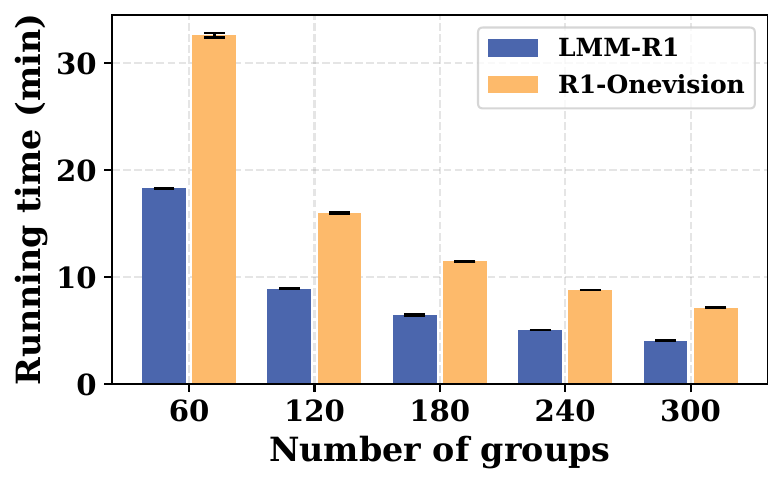} 
        \vspace{-6.5mm}\caption{Group numbers} 
        \label{fig:time_right} 
    \end{subfigure}
    \vspace{-2.3mm} 
    \caption{Analysis of running time.} 
    \label{fig:overall_time_comparison} 
    \vspace{-6mm} 
\end{figure}

\vspace{-1.9mm}
\subsection{Ablation Study}\vspace{-1.5mm}

First, we compare the running time of our proposed method against full fine-tuning (FullFT) and LoRA
\cite{hu2022lora} on CLEVR-Math, with their methods run for 3 epochs. The results presented in Fig.~\ref{fig:time_left} reveal that our method achieves a significantly lower computational cost when adopting the influence function methods. When the number of training data increases, the running time required by ours is substantially less than that of both FullFT and LoRA, demonstrating its superior efficiency over these traditional techniques.

Then, we show the influence of group numbers for the running time of our proposed method using LMM-R1 and R1-Onevision on CLEVR-Math. The results are depicted in Fig.~\ref{fig:time_right}, where we test configurations with the number of groups ranging from 60 to 300. For both models, the running time demonstrates a significant decrease as the number of groups increases. This empirically confirms that our groupwise operation provides substantial computational savings by partitioning groups, which aligns with our methodology. Furthermore, LMM-R1 consistently outperforms R1-Onevision across all tested group sizes, which is expected as LMM-R1 is a smaller model with fewer parameters.

We also empirically examine the influence of the FWER control level $\varepsilon$ on the uncertainty set introduced in Theorem~\ref{thm:CoRAP_ltt}. It controls the confidence level $1-\varepsilon$ that the configurations retained in $\Lambda_{\mathrm{valid}}(\theta)$ satisfy the risk guarantee in Eq.~\eqref{eq:risk_control}. As shown in Table~\ref{tab:threshold_admission}, our results across four different $\varepsilon$ validate on CLEVR-Math and ScienceQA using LMM-R1. The average empirical probability computed over random subsets of the calibration data consistently meets or exceeds the $1-\varepsilon$ target for both datasets (e.g., $0.875$ observed on ScienceQA when $\varepsilon=0.2$). 
Moreover, these results illustrate the trade-off controlled by $\varepsilon$, as allowing a higher FWER by increasing $\varepsilon$ makes the uncertainty set less conservative, which in turn leads to producing modestly smaller prediction set sizes.

\begin{table}[t]
  \centering
  \small
  \setlength{\tabcolsep}{3pt}
  \begin{tabular}{lccc}
    \toprule
    Dataset & Level $\varepsilon$ & Probability  & Prediction set size \\
    \midrule
    \multirow{4}{*}{CLEVR-Math}
      & 0.2 & $0.875\pm 0.125$ & $2.775 \pm 0.821$ \\
      & 0.4 & $0.625\pm 0.183$ & $2.750 \pm 0.429$ \\
      & 0.6 & $0.500 \pm 0.189$ & $2.575 \pm 0.301$ \\
      & 0.8 & $0.500 \pm 0.189$ & $2.525 \pm 0.285$ \\
    \midrule
    \multirow{4}{*}{ScienceQA}
      & 0.2 & $0.875\pm 0.125$ & $6.525 \pm 0.360$ \\
      & 0.4 & $0.625\pm 0.183$ & $6.375 \pm 0.799$ \\
      & 0.6 & $0.500 \pm 0.189$ & $6.025 \pm 0.671$ \\
      & 0.8 & $0.375\pm 0.183$ & $5.850 \pm 0.568$ \\
    \bottomrule
  \end{tabular}\vspace{-2mm}
  \caption{Empirical probability and prediction set size of prediction sets on LMM-R1 across different levels $\varepsilon$.}
  \label{tab:threshold_admission}\vspace{-6mm}
\end{table}

\begin{figure}[H] 
    \centering
    \begin{subfigure}[t]{0.49\columnwidth} 
        \centering
        \includegraphics[width=\linewidth]{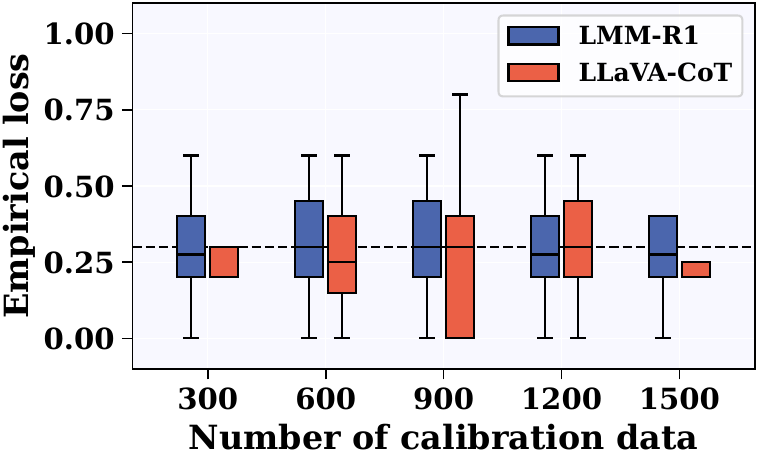} 
        \vspace{-6mm}\caption{Validity} 
        \label{fig:cal_validity} 
    \end{subfigure}
    \hfill 
    \begin{subfigure}[t]{0.49\columnwidth} 
        \centering
        \includegraphics[width=\linewidth]{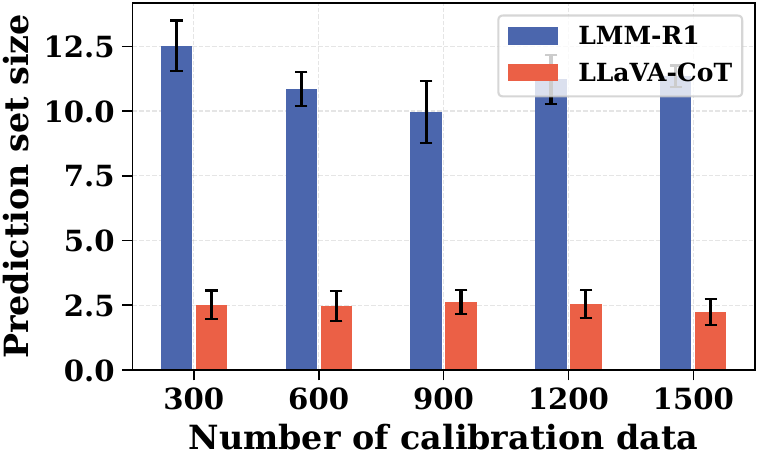} 
        \vspace{-6mm}\caption{Efficiency} 
        \label{fig:cal_eff} 
    \end{subfigure}
    \vspace{-8pt} 
    \caption{Influence of calibration set size for CoRAP.} 
    \label{fig:abl_cal} 
    \vspace{-4mm} 
\end{figure}

We further explore how the calibration set size influences CoRAP in terms of validity and efficiency by varying the number of examples from 300 to 1500 in increments of 300 using LMM-R1 and LLaVA-CoT on ScienceQA. Results in Fig.~\ref{fig:cal_validity} confirm that CoRAP consistently keeps the empirical loss below the target threshold $\alpha = 0.3$. Increasing the calibration size reduces loss variance and indicates improved stability in empirical loss on both models. Regarding the efficiency of our proposed methods, Fig.~\ref{fig:cal_eff} reveals that the prediction set size stabilizes without significant fluctuation. These findings confirm that CoRAP provides uncertainty guarantees while maintaining stable prediction set sizes with limited calibration data.
\begin{figure}[t]
    \centering
    \begin{subfigure}[t]{0.48\linewidth}
        \centering
        \includegraphics[width=\linewidth]{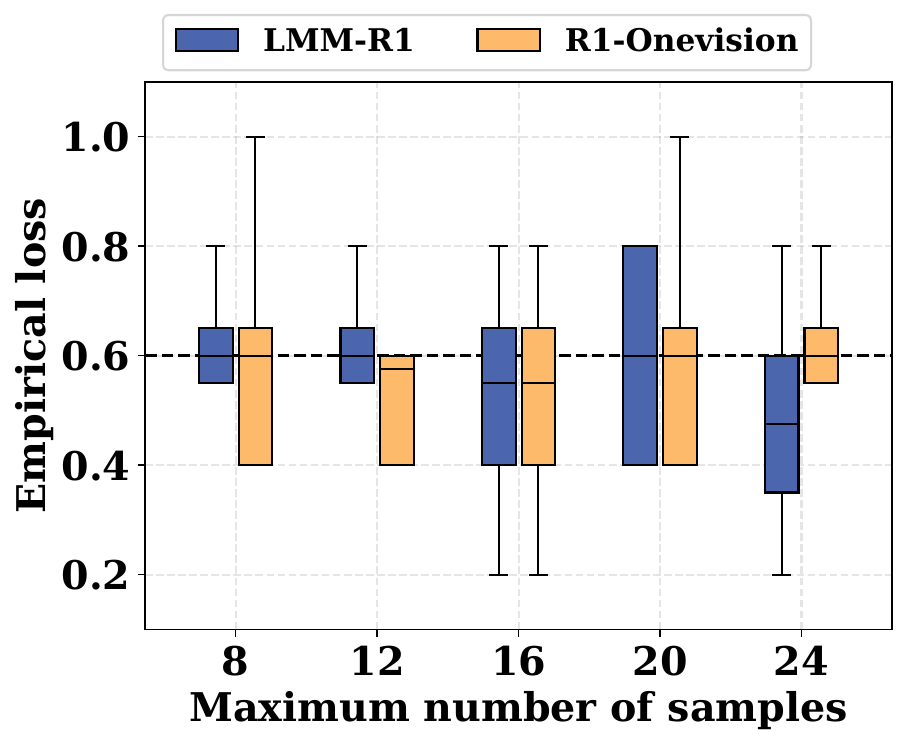}
        \vspace{-6mm}\caption{Validity}
        \label{fig:max_cov}
    \end{subfigure}
    \hfill
    \begin{subfigure}[t]{0.48\linewidth}
        \centering
        \includegraphics[width=\linewidth]{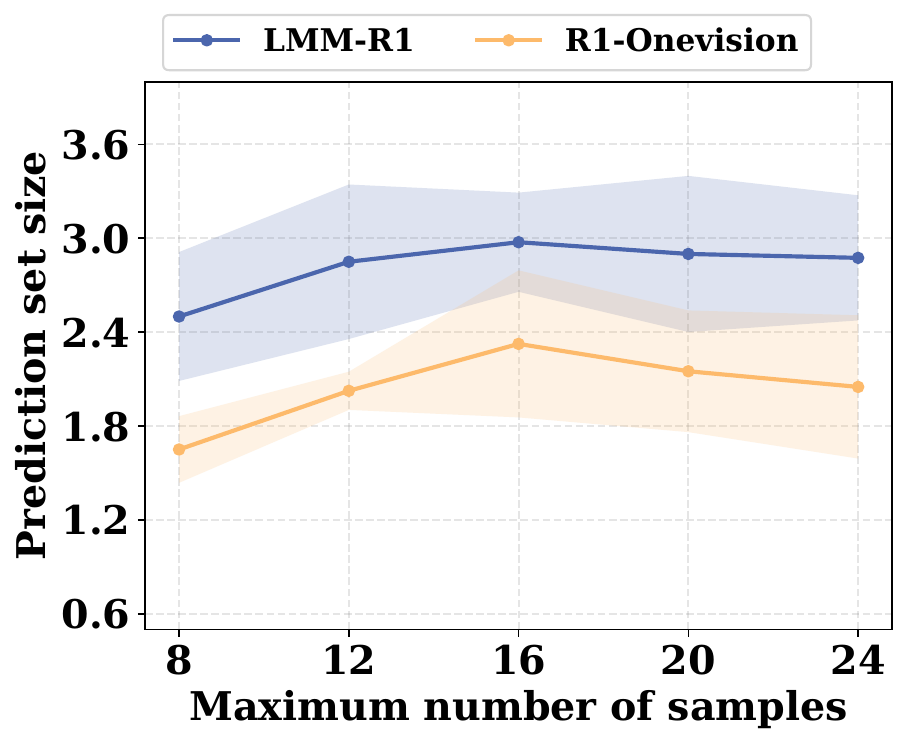} 
        \vspace{-6mm}\caption{Efficiency}
        \label{fig:max_eff}
    \end{subfigure}
    \vspace{-2.5mm}
    \caption{Influence of maximum numbers of samples.}
    \vspace{-7mm}
    \label{fig:max_appendix}
\end{figure}

We further investigate how the maximum sampling budget $K_{\text{max}}$ influences CoRAP's performance. Fig.~\ref{fig:max_cov} confirms that CoRAP strictly adheres to the reliability constraint, with the empirical loss consistently fluctuating near the target $\alpha=0.6$, regardless of the sampling scale. Notably, while LMM-R1 shows slightly higher variance in risk control at $K_{\text{max}}=20$, the overall trend remains well-calibrated. From an efficiency perspective in Fig.~\ref{fig:max_eff}, we observe that the prediction set sizes for both models initially increase and then stabilize as $K_{\text{max}}$ expands. Moreover, R1-Onevision consistently produces more compact prediction sets than LMM-R1, indicating a higher density of correct reasoning paths within its early samples. The convergence of set sizes at higher $K_{\text{max}}$ values demonstrates that CoRAP is computationally economical, providing rigorous uncertainty guarantees without necessitating exhaustive sampling.
\vspace{-1.5mm}
\section{Conclusion}
\label{sec:conclusion}
\vspace{-1.5mm}
In this paper, we design a novel framework to quantify and explain the uncertainty in LRMs with guarantees. To address the challenge of ensuring logical interdependent generations, we first propose CoRAP, which quantifies the uncertainty of the reasoning-answer structure with rigorous guarantees. Subsequently, to resolve the computational complexity in tracing the origins of uncertainty, we design a hierarchical example-to-step explanation framework using Shapley values with guarantees. Furthermore, we conduct the theoretical analysis for our proposed uncertainty quantification and explanation methods. Extensive experiments demonstrate the effectiveness of our approach in ensuring valid and interpretable uncertainty quantification across complex reasoning tasks.

\section{Limitations}
\label{limitations}
\vspace{-1mm}
Our experiments demonstrate that our methods achieve rigorous coverage guarantees and provide provable explanations for uncertainty, confirming the practical efficacy of our framework. However, our current study has several limitations. First, the experimental results are primarily based on specific datasets, so additional experiments on other types of data (e.g., legal or medical) are needed to verify generality. Second, the present study considers multimodal settings using Vision-Language Models. An important next step is to extend our evaluation to text-only settings, applying the framework to standard Large Language Models to assess whether the efficiency and stability gains remain consistent in unimodal tasks.

\section*{Acknowledgments}
\vspace{-1mm}
This work is supported in part by the US National Science Foundation under grants CNS-2350332 and IIS-2442750. Any opinions, findings, and conclusions or recommendations expressed in this material are those of the author(s) and do not necessarily reflect the views of the National Science Foundation. 
\vspace{-1mm}

\bibliography{custom,anthology-1,anthology-2}

\appendix

\section{Algorithm for CoRAP}
\label{sec:algo}
\begin{algorithm}[t]
\caption{Construction of Conformal Reasoning--Answer Prediction (CoRAP) Set}
\label{alg:corap_infer}
\begin{algorithmic}[1]
\Require Input $x$, query $q$, model $\pi_\theta$, functions $Q, F, A$, thresholds $\hat{\lambda}(\theta)=(\hat{\lambda}_1(\theta), \hat{\lambda}_2(\theta), \hat{\lambda}_3(\theta))$, max samples $K_{\max}$
\Ensure Prediction set $\mathcal{C}^{\mathrm{RA}}_{\hat{\lambda}(\theta)}(x,q;\theta)$
\State Initialize $\mathcal{C}^{\mathrm{RA}}_{\hat{\lambda}(\theta)}(x,q;\theta) \gets \emptyset$
\For{$k \gets 1$ \textbf{to} $K_{\max}$}
    \State Sample $\hat{a}_k = \hat{r}_k \Vert \hat{y}_k \sim \pi_\theta(\cdot \mid x,q)$
    \If{$Q(x, q, \hat{a}_k) < \hat{\lambda}_1(\theta)$}
         \State \textbf{continue}
    \EndIf
    \State $\mathcal{C}^{\mathrm{RA}}_{\hat{\lambda}(\theta)}(x,q;\theta) \gets \mathcal{C}^{\mathrm{RA}}_{\hat{\lambda}(\theta)}(x,q;\theta) \cup \{\hat{a}_k\}$
    \If{$\mathcal{C}^{\mathrm{RA}}_{\hat{\lambda}(\theta)}(x,q;\theta) \neq \emptyset$}
        \State $a^* \gets \arg\max_{a \in \mathcal{C}^{\mathrm{RA}}_{\hat{\lambda}(\theta)}(x,q;\theta)} Q(x, q, a)$, where $a^* = r^* \Vert y^*$
        \If{$F(\mathcal{C}^{\mathrm{RA}}_{\hat{\lambda}(\theta)}(x,q;\theta)) \ge \hat{\lambda}_2(\theta)$ \textbf{and} $A(x, q, a^*) \ge \hat{\lambda}_3(\theta)$}
            \State \textbf{break}
        \EndIf
    \EndIf
\EndFor
\State \Return $\mathcal{C}^{\mathrm{RA}}_{\hat{\lambda}(\theta)}(x,q;\theta)$
\end{algorithmic}
\end{algorithm}

In this section, we provide the detailed procedure for CoRAP. As described in Section~\ref{sec:cra}, we discretize the threshold space into a finite grid $\Lambda$. For a fixed model $\theta$ and each candidate threshold $\lambda=(\lambda_1,\lambda_2,\lambda_3)\in\Lambda$, we systematically evaluate its performance on the calibration set $\mathcal D_{\mathrm{cal}}=\{z_i\}_{i=1}^{n_{\mathrm{cal}}}$ with $z_i=(x_i,q_i,r_i,y_i)$ by running Algorithm~\ref{alg:corap_infer} to construct the prediction set $\mathcal C^{\mathrm{RA}}_{\lambda}(x_i,q_i;\theta)$ and computing the binary loss $L^{\mathrm{RA}}(z_i;\lambda,\theta)\in\{0,1\}$, where $1$ indicates miscoverage. To formally certify rigorous risk control, we treat each $\lambda$ as a hypothesis test with the null $H_{0,\lambda}: \mathbb E[L^{\mathrm{RA}}(z;\lambda,\theta)]>\alpha$. Since the total number of failures $S(\lambda):=\sum_{i=1}^{n_{\mathrm{cal}}} L^{\mathrm{RA}}(z_i;\lambda,\theta)$ follows a Binomial distribution under i.i.d. sampling, we compute the Binomial-tail $p$-value $p^{\mathrm{BT}}_{\lambda}=\Pr(\mathrm{Binom}(n_{\mathrm{cal}},\alpha)\le S(\lambda))$ to quantify the likelihood of observing at most $S(\lambda)$ failures if the true risk exceeded $\alpha$. We then apply standard family-wise error rate (FWER) control at level $\varepsilon$ to $\{p^{\mathrm{BT}}_\lambda\}_{\lambda\in\Lambda}$ to obtain a valid subset of thresholds $\Lambda_{\mathrm{valid}}(\theta)\subseteq\Lambda$. If $\Lambda_{\mathrm{valid}}(\theta)$ is empty, the procedure abstains; otherwise, we select an optimal threshold $\hat\lambda(\theta) = (\hat{\lambda}_1(\theta), \hat{\lambda}_2(\theta), \hat{\lambda}_3(\theta))$ based on a chosen efficiency criterion. Finally, as detailed in Algorithm~\ref{alg:corap_infer}, this $\hat\lambda(\theta)$ is subsequently deployed during inference to dynamically filter sampled reasoning-answer pairs via $\hat{\lambda}_1$ and terminate the sampling process early if the set-level and answer-level quality scores satisfy $\hat{\lambda}_2$ and $\hat{\lambda}_3$, respectively, or when the maximum sampling budget $K_{\max}$ is exhausted.

\section{Proofs of Theorems}
\label{sec:proof}

\begin{manualtheorem}{3.1}\label{thm:CRC_app}
Assume calibration examples in $\mathcal D_{\mathrm{cal}}$ and a test sample $z_t$ are i.i.d., and CoRAP's sampling randomness is independent across examples.
Let $\Lambda_{\mathrm{valid}}(\theta)$ be obtained by applying an FWER-$\varepsilon$ procedure to the $p$-values
for target level $\alpha\in(0,1)$ on model $\theta$.
Then, for any
$\hat\lambda(\theta)\in\Lambda_{\mathrm{valid}}(\theta)$, we have the risk constraint
\begin{align}
\mathbb E[L^{\mathrm{RA}}(z_t;\hat\lambda(\theta),\theta)] \le \alpha,
\end{align}
with probability at least $1-\varepsilon$ over $\mathcal D_{\mathrm{cal}}$.
\end{manualtheorem}

\begin{proof}
Fix a model $\theta$ and a finite grid of candidate thresholds $\Lambda$.
For each $\lambda\in\Lambda$, let $R(\lambda;\theta):=\mathbb E[L^{\mathrm{RA}}(z;\lambda,\theta)]$ denote the population miscoverage risk.
On the calibration set $\mathcal D_{\mathrm{cal}}=\{z_i\}_{i=1}^{n_{\mathrm{cal}}}$, define the empirical risk
$\widehat R(\lambda;\theta):=\frac{1}{n_{\mathrm{cal}}}\sum_{i=1}^{n_{\mathrm{cal}}} L^{\mathrm{RA}}(z_i;\lambda,\theta)$
and the count of failures $S(\lambda):=n_{\mathrm{cal}}\widehat R(\lambda;\theta)$.
By the i.i.d.\ assumption and the independence of CoRAP's sampling randomness across examples, for each fixed $\lambda$ we have
$S(\lambda)\sim \mathrm{Binom}(n_{\mathrm{cal}},R(\lambda;\theta))$.

For the target level $\alpha$, consider testing the one-sided null $H_{0,\lambda}: R(\lambda;\theta)>\alpha$.
We use the binomial-tail $p$-value
$p^{\mathrm{BT}}_\lambda := \Pr(\mathrm{Binom}(n_{\mathrm{cal}},\alpha)\le S(\lambda))$.
Under $H_{0,\lambda}$, this $p$-value is super-uniform:
for any $u\in[0,1]$, letting $c(u)$ be the largest integer with
$\Pr(\mathrm{Binom}(n_{\mathrm{cal}},\alpha)\le c(u))\le u$, we have
\begin{align}
\Pr(p^{\mathrm{BT}}_\lambda \le u)
&=\Pr(S(\lambda)\le c(u)) \\
&\le \Pr(\mathrm{Binom}(n_{\mathrm{cal}},\alpha)\le c(u))
\le u,\notag
\end{align}
where the inequality uses that the binomial CDF at a fixed count is non-increasing in the success probability, and
$R(\lambda;\theta)>\alpha$ under the null.

Let $\Lambda_{\mathrm{valid}}(\theta)$ be the set returned by any FWER-$\varepsilon$ procedure applied to
$\{p^{\mathrm{BT}}_\lambda\}_{\lambda\in\Lambda}$.
By FWER control and the super-uniformity above, with probability at least $1-\varepsilon$ over $\mathcal D_{\mathrm{cal}}$,
no $\lambda\in\Lambda_{\mathrm{valid}}(\theta)$ violates the target risk, i.e.,
$R(\lambda;\theta)\le \alpha$ for all $\lambda\in\Lambda_{\mathrm{valid}}(\theta)$.
In particular, this holds for any (possibly data-dependent) choice $\hat\lambda(\theta)\in\Lambda_{\mathrm{valid}}(\theta)$.

Finally, since $z_t$ is an independent draw from the same distribution,
$\mathbb E[L^{\mathrm{RA}}(z_t;\hat\lambda(\theta),\theta)] = R(\hat\lambda(\theta);\theta)$.
With probability at least $1-\varepsilon$ over $\mathcal D_{\mathrm{cal}}$, we have $R(\hat\lambda(\theta);\theta)\le \alpha$,
and hence
\begin{align}
\mathbb E[L^{\mathrm{RA}}(z_t;\hat\lambda(\theta),\theta)] \le \alpha.
\end{align}
The proof is complete.
\end{proof}

\begin{manualtheorem}{3.2}
\label{thm:two_level_exp_app}
Let $z_t$ be a test example that is i.i.d. with the calibration examples in $\mathcal{D}_{\mathrm{cal}}$, and let
$\mathcal H := \{\,L^{\mathrm{RA}}(z_t;\hat\lambda(\theta_0),\theta_0)=1\,\}$
denote the event that the base model $\theta_0$ fails on $z_t$.
Condition on the validity of the thresholds used in evaluating
$v_{\mathrm{ex}}$ and $v_{\mathrm{st}}$.
Assume that for a target level $\alpha \in (0,1)$ and error bounds $\xi_{\mathrm{ex}}, \xi_{\mathrm{st}} \ge 0$, the condition
$
v(\mathcal P\setminus U) \ge v(\mathcal P) -\sum_{u\in U}\phi_u - \xi
$
holds for $(v,\mathcal P,\xi)=(v_{\mathrm{ex}},\mathcal P_{\mathrm{ex}},\xi_{\mathrm{ex}})$
and
$(v,\mathcal P,\xi)=(v_{\mathrm{st}},\mathcal P_{\mathrm{st}},\xi_{\mathrm{st}})$.

\noindent\textbf{(i) (Example-level).}
With probability at least $1-\delta_{\mathrm{ex}}$ over the Shapley MC, the example subset $S^\star_{\mathrm{ex}}$
identified by the stopping condition $\sum_{u\in S^\star_{\mathrm{ex}}}\mathcal B^{\mathrm{ex}}_{u} \ge 1-\alpha$
is sufficient to achieve the risk control:
\begin{align}
\mathbb{E}[
L^{\mathrm{RA}}(z_t;\hat\lambda(\theta^{S^\star_{\mathrm{ex}}}),\theta^{S^\star_{\mathrm{ex}}})
\mid \mathcal H
]
\ \le\ \alpha+\xi_{\mathrm{ex}}.
\end{align}

\noindent\textbf{(ii) (Step-level).}
Conditionally on $S^\star_{\mathrm{ex}}$, with probability at least $1-\delta_{\mathrm{st}}$ over the step-level Shapley MC,
the step subset $T^\star$ identified by the stopping condition $\sum_{u\in T^\star}\mathcal B^{\mathrm{st}}_{u} \ge 1-\alpha$
is sufficient to achieve the risk control:
\begin{align}
\mathbb{E}[
L^{\mathrm{RA}}(z_t;\hat\lambda(\theta^{T^\star}),\theta^{T^\star})
\mid \mathcal H
]
\ \le\ \alpha+\xi_{\mathrm{st}}.
\end{align}

\end{manualtheorem}
\begin{proof}
We condition throughout on the event that the conformal thresholds $\hat\lambda(\theta^S)$ and $\hat\lambda(\theta^T)$
used to evaluate $v_{\mathrm{ex}}(\cdot; z_t)$ and $v_{\mathrm{st}}(\cdot; z_t)$ are valid for all coalitions considered.

We first prove the example-level statement.
Let $\phi^{\mathrm{ex}}_u$ be the Shapley value of the set function $v_{\mathrm{ex}}(\cdot; z_t)$ for each
$u\in\mathcal P_{\mathrm{ex}}$.
By the uniform lower confidence bound, with probability at least
$1-\delta_{\mathrm{ex}}$ over the Shapley Monte Carlo, all bounds hold simultaneously, i.e.,
$\phi^{\mathrm{ex}}_u \ge \mathcal B^{\mathrm{ex}}_u$ for all $u$.
On this event, the stopping rule implies
$\sum_{u\in S^\star_{\mathrm{ex}}}\phi^{\mathrm{ex}}_u \ge \sum_{u\in S^\star_{\mathrm{ex}}}\mathcal B^{\mathrm{ex}}_u \ge 1-\alpha$.
Let $U:=\mathcal P_{\mathrm{ex}}\setminus S^\star_{\mathrm{ex}}$.
By the inequality assumed in Theorem~\ref{thm:two_level_exp}, applied to $v_{\mathrm{ex}}(\cdot; z_t)$,
removing $U$ from the full set can decrease the value by at most $\sum_{u\in U}\phi^{\mathrm{ex}}_u+\xi_{\mathrm{ex}}$.
Using Shapley efficiency (the total Shapley mass equals $v_{\mathrm{ex}}(\mathcal P_{\mathrm{ex}};z_t)-v_{\mathrm{ex}}(\varnothing;z_t)$),
we obtain the lower bound
\begin{align}
\label{eq:proof_vex_lb}
v_{\mathrm{ex}}(S^\star_{\mathrm{ex}};z_t)
\!\ge\!
v_{\mathrm{ex}}(\varnothing;z_t)
+ \!\!\sum_{u\in S^\star_{\mathrm{ex}}}\phi^{\mathrm{ex}}_u
-\xi_{\mathrm{ex}} . 
\end{align}

\noindent On the miscoverage event $\mathcal H$, the empty coalition corresponds to the initial configuration, hence
$v_{\mathrm{ex}}(\varnothing;z_t)=1-L^{\mathrm{RA}}(z_t;\hat{\lambda}(\theta_0),\theta_0)=0$.
Combining this with $\sum_{u\in S^\star_{\mathrm{ex}}}\phi^{\mathrm{ex}}_u\ge 1-\alpha$ yields
$v_{\mathrm{ex}}(S^\star_{\mathrm{ex}};z_t)\ge 1-\alpha-\xi_{\mathrm{ex}}$ on $\mathcal H$.
By definition \eqref{eq:vex}, this implies
$L^{\mathrm{RA}}(z_t;\hat\lambda(\theta^{S^\star_{\mathrm{ex}}}),\theta^{S^\star_{\mathrm{ex}}})\le \alpha+\xi_{\mathrm{ex}}$ on $\mathcal H$,
and taking conditional expectation given $\mathcal H$ gives the desired bound
\begin{align}
\mathbb{E}[
L^{\mathrm{RA}}(z_t;\hat\lambda(\theta^{S^\star_{\mathrm{ex}}}),\theta^{S^\star_{\mathrm{ex}}})
\mid \mathcal H
]
\ \le\ \alpha+\xi_{\mathrm{ex}} .
\end{align}

\vspace{-1mm}
We next prove the step-level statement.
Condition on the realized $S^\star_{\mathrm{ex}}$ and consider the step universe
$\mathcal P_{\mathrm{st}}=\{(j,s): j\in S^\star_{\mathrm{ex}},\ s\in[L_j]\}$.
Let $\phi^{\mathrm{st}}_u$ be the Shapley value of the set function $v_{\mathrm{st}}(\cdot; z_t)$ for each
$u\in\mathcal P_{\mathrm{st}}$.
By the same uniform-LCB argument, with probability at least $1-\delta_{\mathrm{st}}$ over the step-level Shapley Monte Carlo,
we have $\phi^{\mathrm{st}}_u\ge \mathcal B^{\mathrm{st}}_u$ for all $u\in\mathcal P_{\mathrm{st}}$.
On this event, the stopping rule implies
$\sum_{u\in T^\star}\phi^{\mathrm{st}}_u \ge \sum_{u\in T^\star}\mathcal B^{\mathrm{st}}_u \ge 1-\alpha$.

\begin{figure*}[t] 
  \centering
  \begin{subfigure}[t]{0.25\textwidth}
    \centering
    \includegraphics[width=\linewidth]{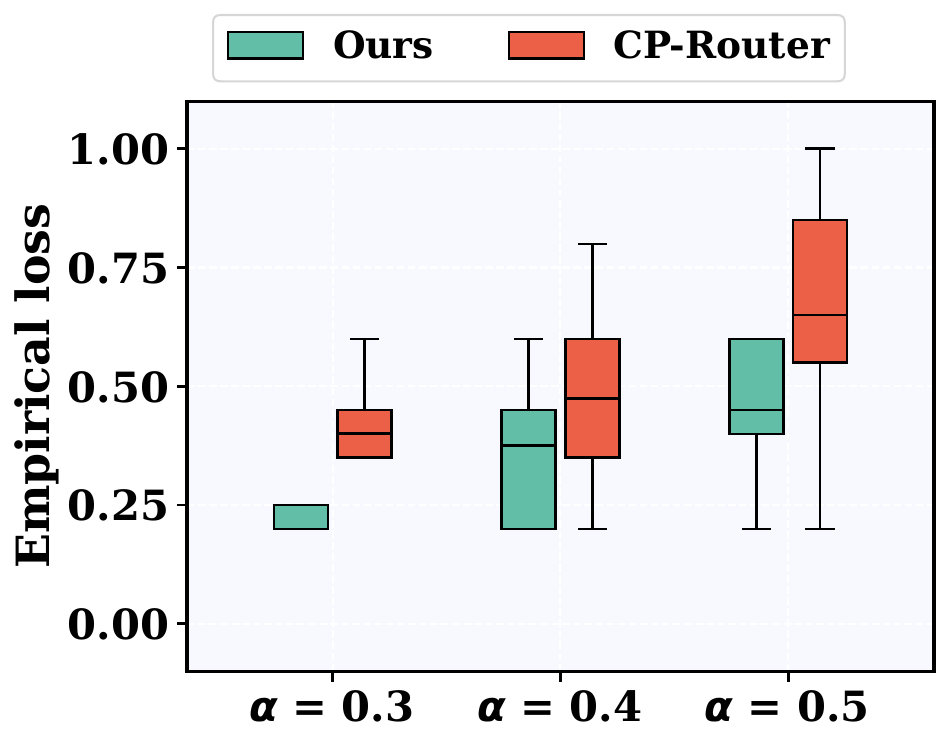}
    \caption{Validity on ScienceQA using LLaVA-CoT }
    \label{fig:cp_validity_llava-cot_scienceqa}
  \end{subfigure}\hfill
  \begin{subfigure}[t]{0.25\textwidth}
    \centering
    \includegraphics[width=\linewidth]{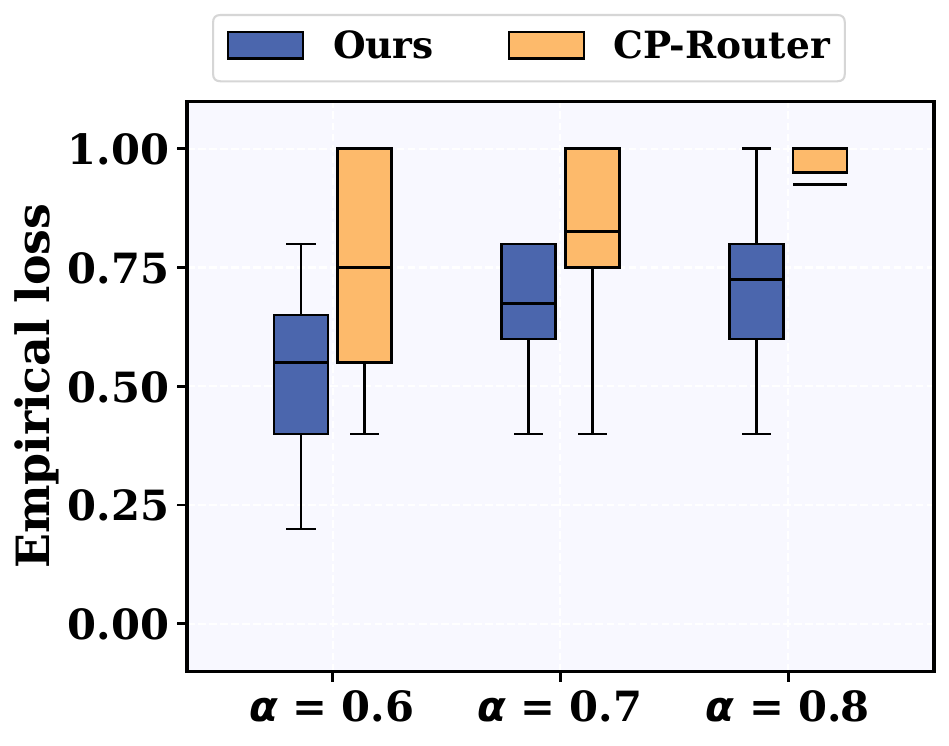}
    \caption{Validity on CLEVR-Math using R1-Onevision}
    \label{fig:cp_validity_r1-onevision_clevr-math}
  \end{subfigure}\hfill
    \begin{subfigure}[t]{0.232\textwidth}
    \centering
    \includegraphics[width=\linewidth]{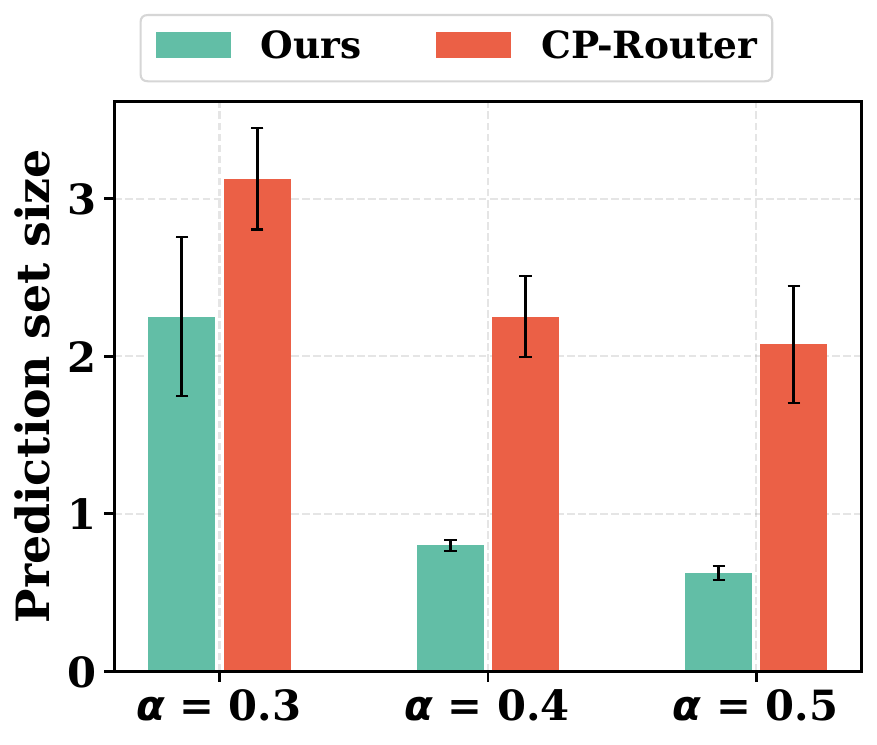}
    \caption{Efficiency on ScienceQA using LLaVA-CoT}
    \label{fig:cp_eff_llava-cot_scienceqa}
  \end{subfigure}\hfill
  \begin{subfigure}[t]{0.232\textwidth}
    \centering
    \includegraphics[width=\linewidth]{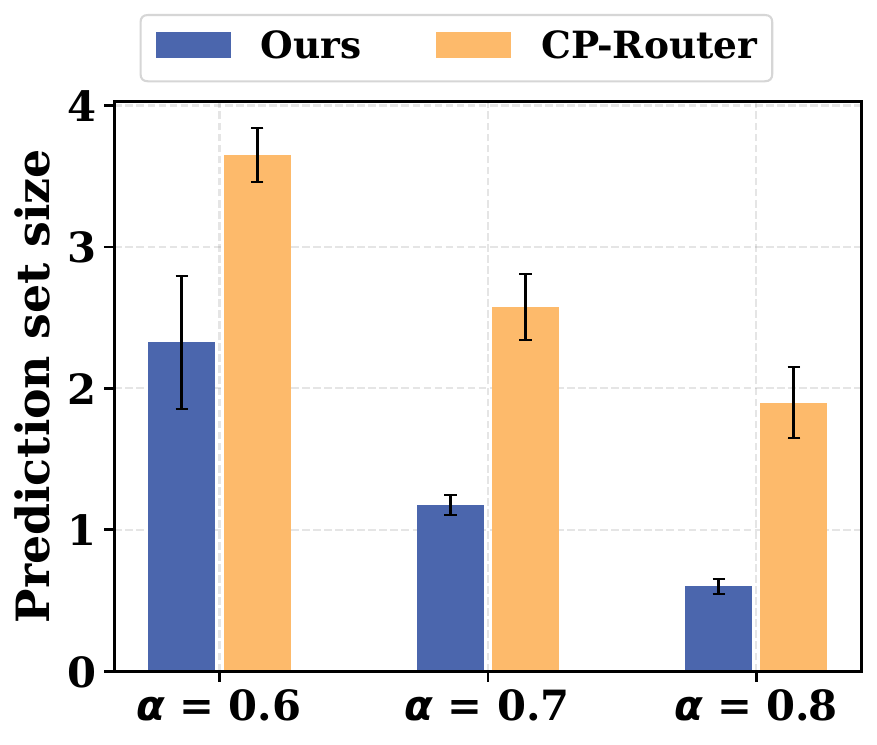}
    \caption{Efficiency on CLEVR-Math using R1-Onevision}
    \label{fig:cp_eff_r1-onevision_clevr-math}
  \end{subfigure}

  \vspace{-8pt}
  \caption{Empirical loss and efficiency results across datasets and models.}
  \label{fig:conformalresults_app}
  \vspace{-6mm}
\end{figure*}

Let $U:=\mathcal P_{\mathrm{st}}\setminus T^\star$.
By the same inequality in Theorem~\ref{thm:two_level_exp}, applied to $v_{\mathrm{st}}(\cdot; z_t)$,
removing $U$ from the full set can decrease the value by at most $\sum_{u\in U}\phi^{\mathrm{st}}_u+\xi_{\mathrm{st}}$.
Using Shapley efficiency gives
\begin{align}
\label{eq:proof_vst_lb}
v_{\mathrm{st}}(&T^\star;z_t)
\ge
v_{\mathrm{st}}(\varnothing;z_t)
+ \sum_{u\in T^\star}\phi^{\mathrm{st}}_u
 - \xi_{\mathrm{st}} .
\end{align}
\vspace{-2mm}

\noindent On $\mathcal H$, the empty step coalition corresponds to training on no selected steps and retains the initial failure,
so $v_{\mathrm{st}}(\varnothing;z_t)=1-L^{\mathrm{RA}}(z_t;\hat{\lambda}(\theta_0),\theta_0)=0$.
Therefore $v_{\mathrm{st}}(T^\star;z_t)\ge 1-\alpha-\xi_{\mathrm{st}}$ on $\mathcal H$.
By definition \eqref{eq:vst}, this implies
$L^{\mathrm{RA}}(z_t;\hat\lambda(\theta^{T^\star}),\theta^{T^\star})\le \alpha+\xi_{\mathrm{st}}$ on $\mathcal H$,
and taking conditional expectation given $\mathcal H$ yields
\vspace{-2mm}
\begin{align}
\mathbb{E}[
L^{\mathrm{RA}}(z_t;\hat\lambda(\theta^{T^\star}),\theta^{T^\star})
\mid \mathcal H
]
\ \le\ \alpha+\xi_{\mathrm{st}} .
\end{align}
The proof is complete.
\end{proof}
\vspace{-2mm}


\vspace{-2mm}
\section{More Experimental Details}
\vspace{-1.5mm}
\label{sec:app_exp_detail}
\textbf{Datasets.} We evaluate our method on two distinct benchmarks. First, CLEVR-Math is a mathematical dataset designed for visual arithmetic reasoning released under the CC BY 4.0 license. It consists of images depicting 3D geometric objects (such as spheres, cubes, and cylinders) paired with questions that require performing mathematical operations based on visual attributes. Second, ScienceQA is a large-scale multimodal benchmark covering natural, social, and language sciences distributed under the CC BY-NC-SA 4.0 license. It is publicly available and derived from standard educational materials. For our experiments, we specifically utilize the curated versions of these benchmarks from~\cite{tan2025reason}, which provide explicit Chain-of-Thought (CoT) annotations essential for LRMs. The utilized data comprises 35000 training samples for CLEVR-Math and 2112 training samples for ScienceQA. Given the synthetic nature of CLEVR-Math and the educational origin of ScienceQA, the data is free from personally identifying information and offensive material.

\textbf{Parameter settings.} In our experimental setup, we adopt 3B LMM-R1~\cite{peng2025lmm}, 7B R1-Onevision~\cite{yang2025r1}, and the 11B LLaVA-CoT~\cite{xu2025llava} models on CLEVR-Math. Instead of full fine-tuning, we approximate parameter updates on the language model head using influence functions with a step size $\eta$ of $5\times 10^{-6}$. On ScienceQA, we similarly adopt the LLaVA-CoT, LMM-R1, and R1-Onevision models. These models utilize identical influence function configurations, applying a step size $\eta$ of $5\times 10^{-6}$ to estimate the contribution of training factors. We set the default level $\varepsilon$ to 0.2. To ensure rigorous explanation guarantees, we perform 256 Monte Carlo permutations with a failure probability $\delta$ of 0.25. For the admission function $V(z_i,\hat{r}_k)$,  we admit a reasoning trace if its ROUGE-L~\cite{lin2004rouge} against the reference rationale is $\ge 0.2$ (chosen empirically).  $V$ can be replaced by other verifiers. We implement our models with PyTorch (v2.5.1) and HuggingFace Transformers (v4.46.1). LoRA fine-tuning uses PEFT (v0.12.0) and ROUGE evaluation uses the official ROUGE implementation.

\textbf{Machine configuration.} The experiments are implemented using the PyTorch framework and run on a cluster equipped with AMD 32-core 2.6GHz CPUs and Nvidia A100 40/80GB GPUs.

\begin{figure}[ht]
    \centering
    \begin{subfigure}[t]{0.48\linewidth}
        \centering
        \includegraphics[width=\linewidth]{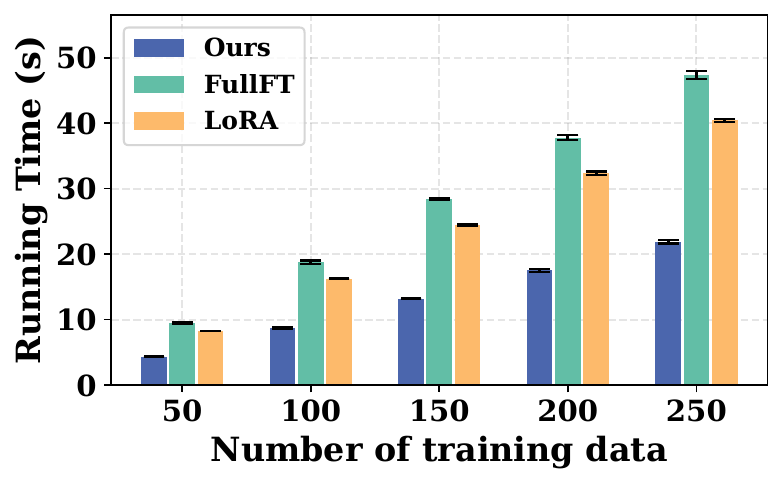}
        \caption{LMM-R1}
        \label{fig:ft_sci_lmmr1}
    \end{subfigure}
    \hfill
    \begin{subfigure}[t]{0.48\linewidth}
        \centering
        \includegraphics[width=\linewidth]{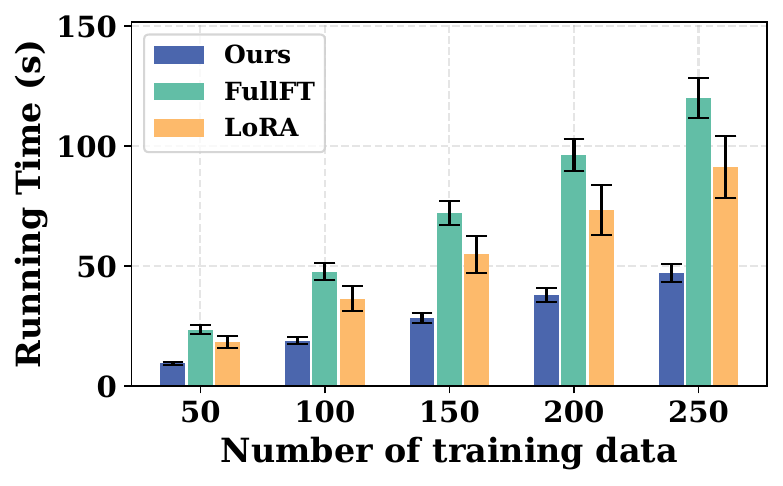} 
        \caption{LLaVA-CoT}
        \label{fig:ft_sci_llavacot}
    \end{subfigure}
    \vspace{-2mm}
    \caption{Influence of fine-tuning methods on ScienceQA.}
    \vspace{-4mm}
    \label{fig:ft_sci}
\end{figure}

\section{More Experimental Results}
\label{sec:app_exp_result}
\vspace{-1mm}
\textbf{Validity.} Figures~\ref{fig:cp_validity_llava-cot_scienceqa} and~\ref{fig:cp_validity_r1-onevision_clevr-math} empirically present the validity of our proposed framework. We employ R1-Onevision on the CLEVR-Math dataset with target significance levels $\alpha$ set to 0.6, 0.7, and 0.8, while for LLaVA-CoT on ScienceQA, we utilize levels of 0.3, 0.4, and 0.5. Across these diverse configurations, the empirical losses of our method consistently stay within the specified thresholds. This confirms that our approach strictly adheres to the risk control requirements defined in Eq.~\eqref{eq:risk_control} and fulfills the theoretical guarantees of Theorem~\ref{thm:CoRAP_ltt}. Conversely, the baseline CP-Router frequently exhibits empirical losses that violate the target $\alpha$ levels, signaling an inability to maintain reliable risk control in complex reasoning scenarios.

\textbf{Efficiency.} Beyond validity, we analyze the compactness of the constructed uncertainty sets in Figure~\ref{fig:cp_eff_llava-cot_scienceqa} and~\ref{fig:cp_eff_r1-onevision_clevr-math}. Under the same experimental settings, our framework demonstrates superior efficiency compared to CP-Router by generating significantly smaller sets. For instance, when applying R1-Onevision to CLEVR-Math at $\alpha=0.7$, our method yields a tight average set size of 1.175, implying that a minimal number of reasoning traces is sufficient to cover the true answer. In comparison, CP-Router produces much looser sets with an average size of 2.575 under identical conditions, resulting in higher redundancy. Overall, these results demonstrate that our CoRAP efficiently and effectively quantifies uncertainty by producing small and valid uncertainty sets.

\textbf{Ablation study.} In Fig.~\ref{fig:ft_sci} we compare the running time of our proposed method against full fine-tuning (FullFT) and LoRA finetuning on the ScienceQA dataset, utilizing both LMM-R1 and LLaVA-CoT models, with the baselines run for 3 epochs. The results reveal that our method achieves a significantly lower computational cost when adopting the influence function methods. As the number of training data increases, the running time required by ours is substantially less than that of both FullFT and LoRA for both architectures, demonstrating its superior efficiency over these traditional techniques across different LRMs.

\vspace{-1mm}
\section{Background}
\label{sec:bg}
\vspace{-1.5mm}

\paragraph{Large reasoning models (LRMs).}
We consider a supervised reasoning task using a dataset $\mathcal{D}_{\mathrm{tr}}=\{(x_j,q_j,r_j,y_j)\}_{j=1}^{n_{\mathrm{tr}}}$, where each instance consists of an input image $x_j$, a query $q_j$, a reasoning sequence $r_j=[r_{j,1},\ldots,r_{j,L_j}]$, and a final answer $y_j$. We define the universe of all reasoning steps across training examples as $\mathcal{U}_{\mathrm{st}}=\{(j,s): j\in[n_\mathrm{tr}],\ s\in[L_j]\}$. An LRM parameterizes a generation policy $\pi_\theta$ that models the joint probability of reasoning and answer tokens conditioned on the input. Representing each sample as a single autoregressive sequence $a_j=[r_{j,1},\ldots,r_{j,L_j},y_j]$, the supervised fine-tuning (SFT) objective maximizes the likelihood of generating the full sequence:
\vspace{-2mm}
\begin{align}
\label{eq:sft}
&\mathcal{L}_{\text{SFT}} = \\
&- \mathbb{E}_{(x,q,r,y)\sim \mathcal{D}_{\mathrm{tr}}} 
\sum_{s=1}^{|a_j|}
\log \pi_{\theta}(a_{j,s}\mid x_j, q_j. a_{j,<s}). \notag
\end{align}

\vspace{-2mm}
\noindent This objective provides the training setup used throughout the paper.

\paragraph{Split conformal prediction.}
Given a model $\theta$, split conformal prediction (CP) provides reliability guarantees by outputting a prediction set $\mathcal{C}_\alpha$, rather than just one answer, which is guaranteed to contain the true answer $y$ with a high and user-specified probability (e.g., $1-\alpha$). This method works for any black-box model $\theta$ by using a separate calibration set $\mathcal{D}_{\mathrm{cal}}=\{(x_i, q_i, r_i, y_i)\}_{i=1}^{n_\mathrm{cal}}$. The key component of split CP is a nonconformity score function $\mathcal{S}((x, q), y;\theta)$. This function measures how unusual or atypical a candidate answer $y$ is for a given input $(x, q)$ according to the model. A high score means the pair $((x, q), y)$ is non-conforming (i.e., the model finds it unlikely), while a low score means it is conforming. To provide its statistical guarantee, split CP requires that the calibration data $\mathcal{D}_{\mathrm{cal}}$ and test data are exchangeable.

To determine the conformal prediction set for a test input $(x_{t}, q_{t})$, we test the nonconformity score for each potential answer $y \in \mathcal{Y}$ against a pre-defined significance level $\alpha \in (0,1)$. The goal is to construct a prediction set $\mathcal{C}_\alpha$ that satisfies the marginal coverage guarantee $\mathbb{P}(y_{t} \in \mathcal{C}_{\alpha} (x_{t}, q_{t}; \theta)) \ge 1 - \alpha$, where $(x_{t}, q_{t}, r_{t}, y_{t})$ is a test point exchangeable with the calibration data.

\begin{theorem}[\cite{vovk2005algorithmic}] \label{thm:CptSetThm}
Assume that examples $z_i=(x_i, q_i, r_i, y_i)$ for $i=1,\cdots,n_\mathrm{cal}$ and $z_t$ are exchangeable. For any nonconformity measure $\mathcal{S}$ and target risk $\alpha \in (0,1)$, define the conformal set at $(x_{t}, q_{t})$ as $\mathcal{C}_{\alpha} (x_{t}, q_{t}; \theta) = \{y \in \mathcal{Y}: \mathcal{S}((x_{t}, q_{t}), y;\theta) \leq \hat{q} \}$, where $\hat{q} = \text{Quantile}(1-\alpha; \{\mathcal{S}((x_i, q_i), y_i;\theta)\}_{i=1}^{n_\mathrm{cal}} \cup \{\infty\})$. Then $\mathcal{C}_{\alpha} (x_{t}, q_{t}; \theta)$ satisfies 
\vspace{-2mm}
\begin{align}
    \mathbb{P}(y_{t} \in \mathcal{C}_{\alpha} (x_{t}, q_{t}; \theta)) \ge 1 - \alpha.
\end{align}
\end{theorem}
\vspace{-2mm}

\noindent In essence, this theorem provides the formal guarantee that the constructed prediction set $\mathcal{C}_{\alpha} (x_{t}, q_{t}; \theta)$ will fail to cover the true answer $y_{t}$ with a probability of at most $\alpha$, provided the calibration and test data are exchangeable.

\vspace{-1mm}
\section{AI Assistance in Writing}
\vspace{-1mm}
During the preparation of this manuscript, we utilized AI tools (e.g., GPT-5.2, Gemini 3 Pro) exclusively for copy-editing purposes, including grammatical correction, phrasing refinement, and readability improvement. These tools were not employed to generate scientific concepts, conduct analyses, or formulate conclusions, nor did they influence the study’s methodology or results. All substantive content was developed solely by the human authors, who retain full responsibility for the final text and editorial decisions.

\end{document}